\documentclass[hidelinks,onefignum,onetabnum]{siamart250211}

\usepackage{amsmath,amssymb,amsfonts}%
\usepackage{graphicx}%
\usepackage{multirow}%
\usepackage{cleveref}
\usepackage{mathrsfs}%
\usepackage[title]{appendix}%
\usepackage{xcolor}%
\usepackage{textcomp}%
\usepackage{manyfoot}%
\usepackage{booktabs}%
\usepackage{algorithm}%
\usepackage{algorithmicx}%
\usepackage{algpseudocode}%
\usepackage{listings}%
\usepackage{setspace}
\usepackage{bm}
\usepackage{dsfont}
\usepackage{shuffle}

\newcommand{\logx}[0]{\ell}
\newcommand{\logy}[0]{\ell'}
\newcommand{\ladj}[2]{L^*_{#1}(#2)}
\newcommand{\ladjmap}[1]{L^*_{#1}}
\newcommand{\radjmap}[1]{R^*_{#1}}
\newcommand{\radj}[2]{R^*_{#1}(#2)}

\newcommand{\srp}[1]{\bm{#1}}
\newcommand{\pab}[2]{#1_{#2}}
\newcommand{\mf}[3]{{\mathbb{#1}}_{#2}^{#3}}

\newcommand{\part}[0]{D}
\newcommand{\parx}[0]{D_x}

\newcommand{\bmx}[0]{W}
\newcommand{\bmy}[0]{V}     
\newcommand{\areax}[0]{A}

\newcommand{\tensora}[0]{\mathbb{A}}
\newcommand{\tensorb}[0]{\mathbb{B}}

\newcommand{\adja}[0]{\varphi}
\newcommand{\adjb}[0]{\psi}
\newcommand{\identity}[0]{\mathds{1}}
\newcommand{\tensor}[1]{\mathbb{#1}}
\newcommand{\tensorprod}[0]{\cdot}
\newcommand{\logsig}[2]{\ell^{#1}_{#2}}

\Crefformat{section}{#2Sec.~#1#3}
\crefname{subsection}{Sec.}{Secs.}
\Crefname{subsection}{Sec.}{Secs.}

\usepackage[colorinlistoftodos]{todonotes}

\newcommand{\statespace}[0]{V}
\newcommand{\secondstatespace}[0]{U}
\newcommand{\Banach}[0]{E}
\usepackage{enumitem} 
\usepackage{amssymb,amsmath}
\usepackage{graphicx}

\newcommand{\hilbert}[1]{H} 
\newcommand{\mixed}[1]{\frac{\partial^2 #1}{\partial s\partial t}}
\newcommand{\single}[2]{\frac{\partial #2}{\partial #1}}
\newcommand{\pathx}[0]{x}

\usepackage[acronym]{glossaries}
\glsdisablehyper
\newacronym{PABs}{}{{piecewise log-linear paths}}
\newacronym{PAB}{}{{piecewise log-linear path}}
\newacronym{PA}{}{{piecewise log-linear}}

\usepackage{lipsum}
\usepackage{amsfonts}
\usepackage{graphicx}
\usepackage{epstopdf}
\ifpdf
  \DeclareGraphicsExtensions{.eps,.pdf,.png,.jpg}
\else
  \DeclareGraphicsExtensions{.eps}
\fi

\newsiamremark{remark}{Remark}
\newsiamremark{hypothesis}{Hypothesis}
\crefname{hypothesis}{Hypothesis}{Hypotheses}
\newsiamthm{claim}{Claim}
\newsiamremark{fact}{Fact}
\crefname{fact}{Fact}{Facts}

\headers{Log-PDE Methods for Rough Signature Kernels}{M. Lemercier, T. Lyons, and C. Salvi}

\title{Log-PDE Methods for Rough Signature Kernels 
\thanks{Preliminary work
\funding{This work was supported in part by EPSRC (NSFC) under Grant EP/S026347/1, in part by The Alan Turing Institute under the EPSRC869 grant EP/N510129/1, the Data Centric Engineering Programme (under the Lloyd’s
Register Foundation grant G0095), the Defence and Security Programme (funded by
the UK Government) and the Office for National Statistics and The Alan Turing Institute (strategic partnership) and in part by the Hong Kong Innovation and Technology Commission (InnoHK Project CIMDA).}}}

\author{Maud Lemercier\thanks{Mathematical Institute, University of Oxford.
  (\email{maud.lemercier@maths.ox.ac.uk}, \email{terry.lyons@maths.ox.ac.uk}).}
\and Terry Lyons\footnotemark[2]
\and Cristopher Salvi \thanks{Department of Mathematics, Imperial College London.
  (\email{c.salvi@ic.ac.uk})}} 

\usepackage{amsopn}

\ifpdf
\hypersetup{
  pdftitle={Log-PDE Methods for Rough Signature Kernels},
  pdfauthor={M. Lemercier, T. Lyons, and C. Salvi}
}
\fi

\begin{document}

\maketitle

\begin{abstract}
Signature kernels, inner products of path signatures, underpin several machine learning algorithms for multivariate time series analysis. For bounded variation paths, signature kernels were recently shown to solve a Goursat PDE. However, existing PDE solvers only use increments as input data, leading to first-order approximation errors. These approaches become computationally intractable for highly oscillatory input paths, as they have to be resolved at a fine enough scale to accurately recover their signature kernel, resulting in significant time and memory complexities. In this paper, we extend the analysis to rough paths, and show, leveraging the framework of smooth rough paths, that the resulting rough signature kernels can be approximated by a novel system of PDEs whose coefficients involve higher order iterated integrals of the input rough paths. We show that this system of PDEs admits a unique solution and establish quantitative error bounds yielding a higher order approximation to rough signature kernels.

\end{abstract}



\section{Introduction}\label{sec1}

Kernels are at the core of several well-established methods for classification \cite{steinwart2008support}, regression \cite{drucker1996support,brouard2016input}, novelty detection \cite{scholkopf2001estimating} and statistical hypothesis testing \cite{gretton2012kernel,wynne2022kernel}. Real-valued kernels, defined as symmetric positive definite functions $\kappa:\mathcal{X}\times\mathcal{X}\to \mathbb R$, arise in Bayesian statistics as covariance functions for Gaussian process priors \cite{williams2006gaussian} over real-valued functions $f:\mathcal{X}\to\mathbb R$. In deep learning, they have played a pivotal role in understanding the large-scale limits of neural networks \cite{lee2017deep, jacot2018neural}. They also underpin the construction of statistical scoring rules and discrepancies for fitting deep generative models \cite{li2017mmd, lorenzothesis, matsubara2022robust}. They are useful in mesh-free methods for solving partial differential equations and inverse problems \cite{chen2021solving,batlle2024kernel}.
All these techniques are transferable to different input spaces in the sense that they can be tailored to different data types by choosing a suitable kernel. A real-valued kernel can always be represented as an inner product $\kappa(x,y)=\langle  \varphi(x), \varphi(y) \rangle$ between some representations of the inputs $x$ and $y$ in a Hilbert space $H$ with inner product $\langle\cdot,\cdot \rangle$ via a feature map $\varphi:\mathcal{X}\to H$. Depending on the problem at hand, better results might be achieved by mapping the inputs to a high-dimensional or even infinite-dimensional feature space. Numerically, this is tractable if the inner product can be obtained without computing every individual coordinate of $\varphi(x)$ and $\varphi(y)$, a strategy which is commonly referred to as a \emph{kernel trick}. 

In this article, we consider so-called \emph{signature kernels}  \cite{kiraly2019kernels, salvi2021rough, salvi2021signature}, a class of symmetric positive definite functions defined on some spaces of paths, known for their effectiveness in time series data analysis. These kernel functions have an explicit representation in terms of the signature, a central map in stochastic analysis \cite{lyons1998differential}. The latter maps any continuous bounded variation path $\pathx:[0,T]\to \statespace$ with values in a vector space $\statespace$ to the solution of a linear controlled differential equation (CDE) 
\begin{align}\label{eq:sig_cde}
dS(x)_t=S(x)_t\tensorprod d\pathx_t, \text{ on }[0,T]
\end{align}
with values in $T((V))  = \prod_{k=0}^\infty V^{\otimes k}$, driven by $x$ and started at $S(x)_0=(1,0,0,\ldots)$. 
\noindent 
A well-known result states that linear functionals on the tensor algebra \( T((V)) \), when restricted to the range of the signature, form a unital algebra that separates points. As a consequence, the classical Stone--Weierstrass theorem implies that such functionals are dense in the space of continuous real-valued functions on compact sets of unparameterized paths~\cite{cass2024topologies}. This endows the signature with strong representational power, allowing for the approximation of general path-dependent functionals through linear models~\cite{salvi2023structure}. In recent years, this theoretical foundation has fueled a surge of interest in signature methods, which have been successfully deployed across a wide spectrum of data science applications. These include deep learning approaches to sequential data~\cite{kidger2019deep, morrill2021neural, salvi2022neural, hoglund2023neural, cirone2024theoretical, issa2024non, barancikova2024sigdiffusions, cirone2025parallelflow}, financial modelling~\cite{arribas2020sigsdes, salvi2021higher, horvath2023optimal, pannier2024path, cirone2025rough}, cybersecurity~\cite{cochrane2021sk}, and computational neuroscience~\cite{holberg2024exact}, among other domains. For a comprehensive and pedagogical introduction, the reader is referred to~\cite{cass_salvi_notes}, and for a concise summary of recent applications, see~\cite{fermanian2023new}. Signature kernels are then defined as \(\kappa(x,y):=\langle S(x), S(y) \rangle\), for various inner products \(\langle \cdot, \cdot \rangle\) defined on \(T((V))\). These have found several applications in Bayesian modelling \cite{toth2020bayesian,lemercier2021siggpde}, regression problems \cite{lemercier2021distribution,salvi2021higher,cochrane2021sk, manten2024signature}, generative modelling \cite{dyer2022approximate,issa2024non}, theoretical deep learning \cite{fermanian2021framing,muca2023neural} and numerical analysis \cite{pannier2024path, shmelev2024sparse}. 


As for many other kernels, signature kernels can be estimated numerically via ``kernel tricks" without explicitly computing the underlying signature feature maps $S(x), S(y)$. Various such signature kernel tricks have been introduced in the literature, including a method based on efficient polynomial evaluation schemes \cite{kiraly2019kernels} as well as a PDE-based approach \cite{salvi2021signature} and variants \cite{cass2024weighted, cass2025numerical}. 

An easy way to obtain a PDE approximation to the signature kernel is to approximate the solution of the CDE \cref{eq:sig_cde} on $[s,t]$ with the solution of the ODE on $T((V))$
\begin{align}\label{eq:local_sig_ode}
    d\tensor{Z}^x_u = \tensor{Z}^x_u \tensorprod x_{s,t}du, \quad u \in [0,1]. 
\end{align}
where $x_{s,t}:=x_t-x_s$ is the increment of $x$. Then, the inner product $f_{u,v}=\langle \tensor{Z}^x_u, \tensor{Z}^y_v\rangle$ where $f : [0,1]^2\to \mathbb R$ satisfies the PDE
\begin{equation}\label{eq:pde_log1}
    \frac{\partial^2 f}{\partial u \partial v}=\langle x_{s,s'}, y_{t,t'}\rangle f, \quad \text{for all }(u,v)\in[0,1]^2.
\end{equation}
All these existing approaches only use increments of the paths $x$ and $y$ as input data and therefore produce error estimates of order $\mathcal{O}(\max\{s'-s, t'-t\})$. 

However, highly oscillatory input paths force the use of small step sizes and the memory and time complexities of PDE solvers for \cref{eq:pde_log1} become prohibitive. Our solution consists of describing highly oscillatory paths as \emph{rough paths} \cite{lyons1998differential}. These map a subinterval of $[0,T]$ to a signature truncated at level $n$, where $n$ is determined by the roughness of the path, and naturally generalise smooth paths which output increments, i.e. signatures truncated at level $n=1$. Taking the first non-trivial case $n=2$ as a running example to ease notation, a geometric $p$-rough path $\mathbb X$, with $p \in [2, 3)$, is then defined on $[s,t]\subset[0,T]$ as $\mathbb{X}_{s,t}=\exp_2\left(\left(0,(x_{s,t}^{(i)})_{i=1}^{d},(x_{s,t}^{(i,j)})_{i,j=1}^{d}\right)\right)$ where $\exp_2$ is the tensor exponential truncated at level $n=2$. In this paper, we observe that \cref{eq:local_sig_ode} is the order $n=1$ case of a more general numerical scheme for CDEs, known as the \emph{log-ODE method}. Applying such method to the signature CDE \cref{eq:sig_cde} yields the order-$n$ ODE
\begin{equation}\label{sig_ode_2}
    d\tensor{Z}^x_u = \tensor{Z}^x_u \tensorprod\log_{n} S(\tensor{X})_{s,t}du, \quad u\in [0,1]
\end{equation}
where $\log_{n}$ is the tensor logarithm truncated at level $n$. Taking inner products of solutions of log-ODEs, we introduce a class of PDE schemes dubbed \emph{log-PDE methods} approximating \emph{rough signature kernels} of the form $\langle S(\mathbb X), S(\mathbb Y)\rangle$. Notably, we prove in \Cref{thm:pde} for $n=2$ that the variables
\begin{align}
    f_{u,v}:= \left\langle \tensor{Z}_u^x, \tensor{Z}_v^y \right\rangle, \quad \varphi^{(i)}_{u,v}:= \left\langle \tensor{Z}_u^x \tensorprod e_i, \tensor{Z}_v^y \right\rangle, \quad \psi^{(i)}_{u,v}:= \left\langle \tensor{Z}_u^x, \tensor{Z}_v^y \tensorprod e_i \right\rangle
\end{align}
satisfy the following system of PDEs
    \begin{align}\label{eq:pde_log2}
    \frac{\partial^2f}{\partial u \partial v}  &= \gamma f + \sum_{j=1}^d \alpha_j \varphi^{(j)} +  \sum_{j=1}^d \beta_j \psi^{(j)} \\
    \frac{\partial \varphi^{(j)}}{\partial v} &= y^{(j)}_{t,t'} f + \sum_{i=1}^d y^{(i,j)}_{t,t'}\psi^{(i)}\\
    \frac{\partial \psi^{(j)}}{\partial u} &= x^{(j)}_{s,s'} f + \sum_{i=1}^d x^{(i,j)}_{s,s'} \varphi^{(i)}, 
\end{align}
for all $(u,v)\in [0,1]\times[0,1]$  with appropriate boundary conditions and where
\begin{equation}
    \alpha_j := \sum_{i=1}^d x^{(i,j)}_{s,s'}y^{(i)}_{t,t'}, \quad \beta_j := \sum_{i=1}^d x^{(i)}_{s,s'}y^{(i,j)}_{t,t'}, \quad 
    \gamma = \sum_{i=1}^{d}x^{(i)}_{s,s'}y^{(i)}_{t,t'}+\sum_{i,j=1}^{d}x^{(i,j)}_{s,s'}y^{(i,j)}_{t,t'}.
\end{equation}
In \Cref{thm:uniqueness} we prove that this system has a unique solution and  in \Cref{thm:pde_genericn} we derive a similar system of PDEs for a generic truncation $n$. Finally, in \Cref{thm:global_error_second_order_pde} we prove that the solution of this system of PDEs provides an approximation to the rough signature kernel  $\langle S(\mathbb X), S(\mathbb Y)\rangle$ of order
\begin{equation}
    \mathcal{O}\left(\max\bigg\{\left(\frac{e\omega_x(0,T)}{N}\right)^{\frac{n+1}{p}-1}\!, \left(\frac{e\omega_y(0,T)}{N'}\right)^{\frac{n+1}{p}-1}\bigg\}\right)
\end{equation}
where $\omega_x$ and $\omega_y$ are controls of $\mathbb X$ and $\mathbb Y$ respectively. Additionally, we show that the signature kernel of smooth rough paths \cite{bellingeri2022smooth} also solve a PDE.

\subsection{Outline}
We begin in \Cref{sec:prelim} by recalling concepts from rough path theory. In \Cref{sec:log_ode_signature} we apply the step-$n$ log-ODE method to the signature CDE. In \Cref{sec:case_two} we present the main results of this paper focusing first on the case $n=2$ for $p$-rough paths with $p \in [1, 3)$. We then extend these results to arbitrary $n\geq \lfloor p \rfloor$ with $p\geq 1$, and to the class of smooth rough paths in \Cref{sec:general}. In \Cref{sec:experiments}, these findings are illustrated on simulated data. Finally, in \Cref{sec:conclusion}, we conclude and outline potential future work directions.


\section{Main definitions}\label{sec:prelim}
In this section, we recall some notions from rough path theory necessary for the rest of the paper. 
\begin{definition}[p-variation]\label{def:pvar} Let $\pathx:[0,T]\to \statespace$ be a continuous path valued in a normed vector space $(\statespace,\|\cdot\|)$. The $p$-variation of $\pathx$ on any interval $[s,t]\subseteq[0,T]$ is defined by
\begin{align*}
\|\pathx\|_{p\mbox{-}\mathrm{var},[s,t]}=\left(\sup_{\part \subset [s,t]}\sum_{t_i\in \part}\|\pathx_{t_{i+1}}-\pathx_{t_{i}}\|^p\right)^{1/p}
\end{align*}
where the supremum is taken over all partitions $\part$ of the interval $[s,t]$. We denote the space of continuous paths of bounded $p$-variation by $C^{p\mbox{-}\mathrm{var}}([0,T],V)$.
\end{definition}

Any continuous path $\pathx:[0,T]\to \statespace$ of finite $1$-variation can be canonically lifted to a path $\tensor{Z}:t\mapsto S(\pathx)_{0,t}$ with values in $T((\statespace))$, the space of tensor series over $\statespace$, simply by considering all its iterated (Riemann-Stieltjes) integrals. This space is endowed with two internal operations: an addition and a product. For any two elements $\tensor{A} = (a^0, a^1,a^2, \ldots)$ and $\tensor{B} = (b^0, b^1, b^2, \ldots)$ in $T((\statespace))$ and any scalar $\lambda\in\mathbb{R}$, 
\begin{gather*}
    \lambda \tensor{A}+\tensor{B}=(\lambda a^0+b^0,\!\;\lambda a^1+b^1,\!\;\ldots) \\ 
    \tensor{A}\tensorprod \tensor{B}= (c^0,\!\;c^1,\!\;\ldots) \quad\text{ with } c^n = \sum_{k=0}^{n}a^k\otimes b^{n-k}.
\end{gather*}
The space $T((\statespace))$ endowed with these operations is a (non-commutative) algebra with unitary element $\identity = (1, 0, 0,\ldots)$. It is often important to look only at finitely many terms of an element of $T((\statespace))$. To this aim, we define $T^n(\statespace)$ the truncated tensor algebra over $\statespace$ of order $n\in\mathbb{N}$. 
More precisely,  $T^n(\statespace)$ is the quotient of $T((\statespace))$ by the ideal $T^{>n}(\statespace)$ defined by
\begin{align*}
T^{>n}(\statespace)=\left\{\tensora=(0,0,\ldots, a^{n+1}, \ldots)~|~ \tensora\in T((\statespace))\right\}.
\end{align*}
We denote by $\pi_{\leq n}$ the quotient map from $T((\statespace))$ to $T^n(\statespace)$, and we identify $T^n(\statespace)$ with $\bigoplus^n_{k=0}\statespace^{\otimes k}$ equipped with the product $$(a^0,a^1,\ldots,a^n)\otimes_n(b^0,b^1,\ldots,b^n)=(c^0,c^1,\ldots,c^n)$$ with $c^i=\sum_{k=0}^{i}a^k\otimes b^{i-k}$. With this identification $\pi_{\leq n}$ becomes a projection. We denote by $\pi_n$ the projection on $\statespace^{\otimes n}$.

\medskip 

The unital associative algebra $(T((\statespace)),+,\tensorprod)$ carries a Lie bracket $[\cdot,\cdot]$ defined by $[\tensora, \tensorb] = \tensora\tensorprod \tensorb - \tensorb\tensorprod \tensora$, and there are several canonical Lie algebras associated to $T((\statespace))$. 

\medskip
\begin{definition}[Lie series]
The space of Lie formal series over $\statespace$, denoted as $\mathcal{L}((\statespace))$ is defined as the following subspace of $T((\statespace))$
\begin{align*}
  \mathcal{L}((\statespace))=\left\{\tensor{L}=(l^0,l^1,\ldots)~|~\forall k\geq 0,~l^k\in L_k\right\}  
\end{align*}
where $L_0 = 0$, $L_1 = \statespace$, and $L_{k+1} = [\statespace, L_k]$, with $[\statespace,\secondstatespace]$ denoting the linear span of all elements of the form $[e,f]$ where $(e,f)\in \statespace\times \secondstatespace$ for any two linear subspaces $\statespace,\secondstatespace$  of $T((\statespace))$.
\end{definition}

\medskip
\noindent If $\pathx$ is a path segment, then $\log S (\pathx) \in \mathcal{L}((\statespace))$ \cite{chen1957integration} where the
logarithm map is defined for any $\tensora\in T^{>0}(\statespace)$ by 
 \begin{align}\label{eq:log}
\log(\identity+\tensora) = \sum_{k\geq 1}\frac{(-1)^{k-1}}{k}\tensora^{\otimes k}.     
\end{align}

\noindent We denote the space of Lie polynomials by $\mathcal{L}(\statespace)$. For any $n \geq 1$, the \emph{step-$n$ free Lie algebra} is defined by $\mathcal{L}^n(\statespace) := \pi_{\leq n}\left(\mathcal{L}((\statespace))\right)$ with elements called Lie polynomials of degree $n$. The map $\log_n$ associates to each $\tensora\in \pi_{\leq n}(T^{>0}(\statespace))$ the element of $T^n(\statespace)$ defined as 
\begin{align*}
\log_n(\identity+\tensora) = \sum_{k=1}^{n}\frac{(-1)^{k-1}}{k}\tensora^{\otimes k}.
\end{align*}
Note that it satisfies $\pi_{\leq n}(\log(\identity+\tensora))=\log_n(\pi_{\leq n}(\identity+\tensora))$. Finally, we note that the path signature $\tensor{Z}$ actually takes its values in a curved subspace $G(\statespace)\subset T((\statespace))$ with a group structure. It is given by $G(\statespace)=\exp \{\mathcal{L}((\statespace))\}$
where 
\begin{align*} \exp(\tensora)=\sum_{k=0}^{\infty}\frac{\tensora^{\otimes k}}{k!}.
\end{align*}
For each $n \geq 1$, we denote $G^n(\statespace)=\pi_{\leq n}(G(\statespace))$ and $\exp_n(\tensora)=\sum_{k=0}^{n}\frac{\tensora^{\otimes k}}{k!}$.

\medskip
\subsection{Rough paths}
More generally, a path of bounded $p$-variation with values in $G^{\lfloor p \rfloor}(\statespace)$ has a canonical lift to a path with values in $G(\statespace)$. In the sequel, denote by $\Delta_T$ the simplex $\Delta_T:=\{(s, t) \in
[0, T]^2~|~0 \leq s \leq t \leq T \}$ and we say that a continuous map $\omega:\Delta_T\to [0,+\infty)$ is a \emph{control function} if it is super-additive, that is, $\omega(s,t)+\omega(t,u)\leq \omega(s,u)$ for all $0\leq s\leq t\leq u \leq T$.

\bigskip 
\begin{definition}[Multiplicative functional] Let $n \geq 1$ be an integer and
let $\mf{X}{}{}:\Delta_T \to T^{n}(\statespace)$ be a continuous map. For each $(s, t) \in \Delta_T$, denote
by $\mf{X}{s,t}{}$ the image by $\mf{X}{}{}$ of $(s, t)$ and write
$$\mf{X}{s,t}{} = (x_{s,t}^{0}, x_{s,t}^{1}, \ldots, x_{s,t}^{n})\in T^{n}(\statespace).$$
The function $\mf{X}{}{}$ is called a multiplicative functional of degree $n$ in $\statespace$ if 

\begin{enumerate}[label=(\roman*)]
\item  $x_{s,t}^0=1$ for all $(s, t) \in \Delta_T$ and 
\item Chen's identity holds, that is, 
$$\mf{X}{s,u}{}\tensorprod\mf{X}{u,t}{} = \mf{X}{s,t}{}, \quad \forall s,t,u\in [0,T], \quad s\leq u\leq t.$$
\end{enumerate}
\end{definition}

\noindent Given a path $\tensor{X}_t=(1,x^1_t, \ldots, x^n_t)$ in $T^n(\statespace)$ we say that $\mf{X}{s,t}{}=\tensor{X}_s^{-1}\tensorprod \tensor{X}_t$ is the multiplicative functional determined by $\tensor{X}$. Conversely, given a multiplicative functional $\mf{X}{s,t}{}$ and a point $\tensor{X}_0\in
T^{n}(\statespace)$, we say that $\tensor{X}_t = \tensor{X}_0\tensorprod \mf{X}{0,t}{}$
is the path starting at $\tensor{X}_0$ determined
by $\mf{X}{s,t}{}$. The $p$-variation over $[s,t]$ of a multiplicative functional $\mf{X}{}{}$ of degree $\lfloor p\rfloor $ is defined by 
\begin{align}
\|\mf{X}{}{}\|_{p\mbox{-}\text{var},[s,t]}:=\sup_{\part\subset [s,t]}\max_{1\leq k\leq \lfloor p\rfloor}\left(\sum_{t_i\in \parx}\|x^k_{s,t}\|^{p/k}\right)^{1/p}   
\end{align}
where the supremum is taken over all finite partitions $\part$ of $[s,t]$.
\bigskip 

\begin{definition}[$p$-rough path] Let $p \geq 1$ be a real number. A $p$-rough path $\mf{X}{}{}$ in $\statespace$ is a multiplicative functional of degree $\lfloor p\rfloor$ in $\statespace$ with finite $p$-variation.
\end{definition}

\bigskip 

\begin{theorem}[Thm. 2.2.1 in \cite{lyons1998differential}]\label{thm:extension}
Let $p\geq 1$ be a real number and $n \geq 1$
an integer. Let $\mf{X}{}{}:\Delta_T\to T^{n}(\statespace)$ be a multiplicative functional with finite
$p$-variation controlled by a control $w$ and assume that $n \geq \lfloor p\rfloor $. Then there exists
a unique extension of $\mf{X}{}{}$ to a multiplicative functional 
$\Delta_T \to T((\statespace))$ which
possesses finite $p$-variation.
More precisely, for every $m\geq p + 1$, there exists a unique continuous
function $x^m:\Delta_T\to \statespace^{\otimes m}$ such that the map
$$(s,t) \mapsto (1,x^1_{s,t}, \ldots,x^{\lfloor p\rfloor }_{s,t}, \ldots, x^m_{s,t},\ldots)\in T((\statespace))$$
is a multiplicative functional with finite $p$-variation controlled by $\omega$. By this
we mean that $\|x^i_{s,t}\|_{\statespace^{\otimes k}}<\frac{\omega(s,t)^{i/p}}{\beta_p(i/p)!}$ $\forall i\geq 1$ $\forall(s,t)\in\Delta_T$
where $\beta_p=p^2\Big(1+\sum_{r=3}^{\infty}\frac{2}{r-2}^{(\lfloor p\rfloor+1)/p}\Big)$.
\end{theorem}

\begin{definition}[Geometric $p$-rough path]
    A geometric $p$-rough path is a $p$-rough path expressed as the limit in $p$-variation of a sequence $(\pi_{\lfloor p\rfloor}(S(x^N)))$ of truncated signatures of bounded variation paths $(x^N)$.
\end{definition}
The space of geometric $p$-rough paths in $\statespace$ is denoted by $G\Omega_p(\statespace)$. We call the unique extension of a geometric $p$-rough path $\tensor{X}\in G\Omega_p(\statespace)$  its signature $$S(\mf{X}{}{}):(s,t)\mapsto (1,x^1_{s,t}, \ldots,x^{\lfloor p\rfloor }_{s,t}, \ldots, x^m_{s,t},\ldots).$$

\medskip
If $\statespace$ is endowed with an inner product $\langle \cdot, \cdot\rangle_1$, and we denote by $\langle \cdot, \cdot\rangle_k$ the canonical (Hilbert-Schmidt) inner product on $\statespace^{\otimes k}$ derived from $\langle \cdot, \cdot\rangle_1$, then an inner product on $T(\statespace)$---the subalgebra of $T((\statespace))$ in which all but finitely many projections are zero---can be defined for any $\tensora, \tensorb \in T(\statespace)$ by
\begin{align*}
    \langle \tensora, \tensorb \rangle := \sum_{k=0}^{\infty}\langle a^k, b^k\rangle_k.
\end{align*}
We note that other choices are possible \cite{cass2024weighted} but will not be considered in this article. With this, a kernel on path space can then be defined as follows.


\medskip\begin{definition}[Signature kernel] Let $p\geq 1$ and $q\geq 1$ be two real numbers. The signature kernel is the map defined for any two geometric $p$- and $q$-rough paths $\mf{X}{}{},\mf{Y}{}{}$ controlled by $w_x$ and $w_y$ respectively, by
    \begin{align}
    \kappa\left(\mf{X}{}{}, \mf{Y}{}{}\right) = \left\langle S(\mf{X}{}{}), S(\mf{Y}{}{})\right\rangle
    \end{align}
\end{definition}
\noindent This kernel is well-defined since $\langle S(\mf{X}{}{}),S(\mf{Y}{}{})\rangle = \sum_{k=0}^{\infty} \langle x^k_{s,t} ,y^k_{u,v}\rangle_{k}$ which is bounded by $\sum_{k=0}^{\infty} \|x^k_{s,t}\|_{\statespace^{\otimes k}}\|y^k_{u,v}\|_{\statespace^{\otimes k}}\leq \sum_{k=0}^{\infty} \frac{\omega_{x}(s,t)^{k/p}}{\beta_p(k/p)!}\frac{\omega_{y}(u,v)^{k/q}}{\beta_q(k/q)!}<+\infty$. 

\medskip
\section{The log-ODE method for the signature RDE}\label{sec:log_ode_signature}

The log-ODE method is an effective method for approximating the solution of an RDE by reducing it to an ODE. This method is based on the log-signature transform, which is the image of a signature under the logarithm map \cref{eq:log}. In this section, we apply the log-ODE method to approximate the signature of a path, and provide error estimates first for paths of bounded variation in \Cref{ssec:smooth} and for geometric $p$-rough paths in \Cref{ssec:rough}.

In the sequel, we consider the Banach algebra $\Banach \subset T((\statespace))$ defined as follows 
\begin{equation*}
    \Banach=\left\{  \tensora\in T((\statespace)) : \| \tensora \|_1 < \infty\right\},
\end{equation*}
where $\|\cdot\|_1$ is $\ell_1$-norm: $\|\tensora\|_1 = \sum_{k=0}^\infty\|a^k\|_{\statespace^{\otimes k}}$. In particular, $\|\tensora\tensorprod\tensorb\|_1 \leq \|\tensora\|_1\|\tensorb\|_1$, $\forall \tensora,\tensorb \in \Banach$.
We will also consider the $\ell_2$-norm: $\|\tensora\|_2 = \sqrt{\sum_{k=0}^\infty\|a^k\|^2_{\statespace^{\otimes k}}}$. Clearly, one has $\|\tensora\|_2 \leq \|\tensora\|_1$. To simplify notation, in the sequel, we write $\log_n(\mathbb A)=\log_n(\pi_{\leq n}\mathbb A)$.
\medskip
\subsection{Smooth case}\label{ssec:smooth}
Let $x \in C^{1\text{-var}}([0,T], \statespace)$ be a continuous path of bounded variation. By standard Picard's arguments, the $\Banach$-valued CDE
\begin{align}\label{ode}
    d\tensor{Z}_t = \tensor{Z
}_t \tensorprod dx_t, \quad \text{started at } \tensor{Z}_s = \tensora \in \Banach,
\end{align}
has a unique solution which coincides with $\tensora \tensorprod S(x)_{s,t}$. The step-$n$ log-ODE method for the CDE (\ref{ode}) on $[s,t]$ reads as
\begin{align}\label{eq:local_sig_ode_dup}
    d\tensor{Z}_u = \tensor{Z}_u \tensorprod \log_nS(x)_{s,t}du, \quad \text{started at } \tensor{Z}_0 = \tensora, 
\end{align}
for all $u \in [0,1]$, with explicit solution given by $\tensor{Z}_1 = \tensora \tensorprod \exp(\log_n S(x)_{s,t})$. The following Lemma, proved in \Cref{ssec:local} states a bound on the local error. 
\begin{lemma}\label{lemma:local_error_smooth}
    Let $x \in C^{1\text{-var}}([0,T], \statespace)$ be a path of bounded variation and let $\mathcal{L}_{s,t}^n(\tensora;x)$ be the solution of the ODE (\ref{eq:local_sig_ode_dup}). Assume that $\|x\|_{1\text{-var},[s,t]} < 1/e$.  Then there exists a constant $C = C(\|\tensora\|_1, \|x\|_{1\text{-var}})$ such that
    \begin{align}
        \|\tensora\tensorprod S(x)_{s,t} - \mathcal{L}_{s,t}^n(\tensora;x)\|_1 \leq C (e\|x\|_{1\text{-var},[s,t]})^{n+1}.
    \end{align}
\end{lemma}

\begin{remark}\label{remark:linear_paths}
   Taking the inner product between the solutions of two step-$1$ log-ODEs (\ref{eq:local_sig_ode_dup}) driven by two bounded variation paths $x$ and $y$ on $[s,s']$ and $[t,t']$ respectively, we obtain that the function defined by $f_{u,v}=\langle \tensor{Z}^x_u, \tensor{Z}^y_v\rangle$ satisfies the PDE 
\begin{equation}\label{eq:pde_log1_dup}
    \frac{\partial^2 f}{\partial u \partial v}=\gamma f, \quad \text{for all }(u,v)\in[0,1]^2,
\end{equation}
where $\gamma:= \sum_{i=1}^{d}x^{(i)}_{s,s'}y^{(i)}_{t,t'}$ with $x^{(i)}_{s,s'}:=x_{s'}^{(i)}-x_s^{(i)}$ is the $i$-th coordinate of the path $x$ on the interval $[s,s']$. This PDE is of the form of the Goursat PDE introduced in \cite{salvi2021signature}. When the two ODEs are started at the identity of $T((V))$, that is when $\mathbb{Z}^x_0=\mathbb{Z}^y_0=\mathds{1}$, the solution of \cref{eq:pde_log1_dup} coincides with the signature kernel of the two linear paths with increments $x^{(i)}_{s,s'}$ and $y^{(i)}_{t,t'}$. 
\end{remark}

\medskip
Given a partition $0=t_0<t_1<\ldots<t_N=T$ of the time interval $[0,T]$, the approximation scheme based on the log-ODE method consists in solving iteratively ODEs of the form (\ref{eq:local_sig_ode_dup}) over the successive time intervals in the partition, i.e. iterating for all $i=0, \ldots, N-1$, the map $$\tensora_{i+1}=\mathcal{L}^{n}_{t_i,t_{i+1}}(\tensora_i;x), \quad\text{started at } \tensora_0=\tensora.$$  Although we have stated only a lemma for the local error so far, a global rate of convergence for the aforementioned scheme will follow by “patching together” these local error estimates over the partition of $[0,T]$. We will carry out the detailed derivation of this global error estimate in the next section, where we extend the scheme to arbitrary geometric $p$-rough paths. In this context, the error estimates will be expressed in terms of a control function, and we emphasize that one can recover the bounded variation case simply by setting $p=1$.

\subsection{Rough case}\label{ssec:rough}
\noindent 
Let $p \geq 1$ and let $\tensor{X} \in G\Omega_p(\statespace)$ be a geometric $p$-rough path. Let $(x^N)$ be a sequence of continuous bounded variation paths such that $(\pi_{\leq \lfloor p \rfloor} S(x^N))$ converges in $p$-variation to $\tensor{X}$. Then, by \cite[Theorem 10.57]{friz2010multidimensional}, there exists a unique $\tensor{Z} \in C^{p\text{-var}}([0,1], \Banach)$ such that the sequence $(\tensor{Z}^N)$ in $C^{1\text{-var}}([0,1],\Banach)$ defined by
\begin{equation}\label{seq_ode}
    d\tensor{Z}^N_t = \tensor{Z}^N_t \tensorprod dx^N_t, \quad \text{started at } \tensora \in \Banach,
\end{equation}
converges uniformly to $\tensor{Z}$ on $[0,1]$ as $N \to \infty$. We say $\tensor{Z}$ is the solution of the RDE
\begin{equation}\label{eq:rde}
    d\tensor{Z}_t = \tensor{Z}_t \tensorprod d\tensor{X}_t, \quad \text{started at } \tensora \in \Banach.
\end{equation}
This limit $\tensor{Z}$ coincides with the unique extension $S(\tensor{X})$ of the rough path $\tensor{X}$ to $T((\statespace))$ as given by \cite[Theorems 3.7]{lyons2004differential}.  In particular, if $\tensor{X}$ is controlled by a control $\omega$, then for any $i \in \mathbb N$ 
\begin{equation}\label{p-var_bound}
\|S(\tensor{X})^i_{s,t}\|_{\statespace^{\otimes i}} \leq \frac{\omega(s,t)^{i/p}}{\beta_p(i/p)!}, \quad \text{where } \beta_p \;=\; p^2 \biggl( 1 \;+\; \sum_{r=3}^{\infty} \bigl(\tfrac{2}{r-2}\bigr)^{\frac{|p|+1}{p}}\biggr).
\end{equation}
As in the smooth case, the step-$\lfloor p \rfloor$ log-ODE method approximation of the solution of the RDE (\ref{eq:rde}) is obtained by solving the ODE 
\begin{equation}\label{rough_logode}
    d\tensor{Z}_u = \tensor{Z}_u \tensorprod\log_{\lfloor p \rfloor} S(\tensor{X})_{s,t}du, \quad \text{started at } \tensor{Z}(0) = \tensora,
\end{equation}
for $u \in [0,1]$. The approximation is given by $\tensor{Z}_{1}$, which is denoted by $\mathcal{L}^{\lfloor p \rfloor}_{s,t}(\tensora; \tensor{X})$. 
\begin{lemma}\label{lemma:local_error}
Let $p \geq 1$ and assume that $(e\omega(s,t))^{1/p} < 1$. Then, there exists a constant $C = C(p, \|\tensora\|_1, \omega(s,t))$ such that   
\begin{align}
    \|\tensora\tensorprod S(\tensor{X})_{s,t} - \mathcal{L}^{\lfloor p \rfloor}_{s,t}(\tensora; \tensor{X})\|_1 &\leq C(e\omega(s,t))^{(n+1)/p}
\end{align}
\end{lemma}
The proof of \Cref{lemma:local_error} is given in \Cref{ssec:local}.
\medskip 

Let $0=t_0<t_1<\ldots <t_N=T$ be a partition of $[0,T]$. Define approximations $\widetilde{\tensor{Z}}_{i}$ for $\tensor{Z}_{t_i}$ in the following way. Start with $\widetilde{\tensor{Z}}_{0}=\tensora$ and for every step, solve the ODE
\begin{equation*}
    d\tensor{Z}_u = \tensor{Z}_u \tensorprod\log_n S(\mf{X}{}{})_{t_i, t_{i+1}}du, \quad \text{started at } \tensor{Z}(0) = \widetilde{\tensor{Z}}_{i}
\end{equation*}
for $u \in [0,1]$. Let $\widetilde{\tensor{Z}}_{i+1}=Z(1)$. The ODE has an explicit solution given by
\begin{equation*}
   \tensor{Z}(1) = \widetilde{\tensor{Z}}_{i}\tensorprod \exp(\log_n S(\mf{X}{}{})_{t_i, t_{i+1}}).
\end{equation*}
Repeat the ODE approximations
successively over each $[t_i, t_{i+1}]$ to get
\begin{equation*}
\widetilde{\tensor{Z}}_{N}=\prod_{i=0}^{N-1}\exp(\log_nS(\mf{X}{}{})_{t_i, t_{i+1}}).
\end{equation*}
We prove in \Cref{ssec:global} the following bound on the global error.
\begin{lemma}\label{lemma:global_error} Let $p \geq 1$ be real number. Consider an integer $n\geq \lfloor p\rfloor$ and a partition $\{0=t_0< t_1< \ldots< t_N=T\}$ of $[0,T]$ such that $(e\omega(t_i,t_{i+1}))^{1/p} < 1$ for all $i=0, \ldots, N-1$. Then, there exists a constant $C = C(\omega(0,T),p)$ such that   
\begin{align}
    \|S(\tensor{X})_{0,T} - \widetilde{\tensor{Z}}_{N}\|_1 &\leq  C\max_{i=0, \ldots, N-1} (e\omega(t_i,t_{i+1}))^{(n+1)/p-1}
\end{align}
\end{lemma}

As we did in \Cref{remark:linear_paths} in the case where $p<2$ and $n=1$, in view of constructing approximation schemes for the signature kernel, we will now consider the function given by the inner product of two step-$n$ log-ODE approximations of the signatures of two geometric $p$-rough paths. The essential question is: does this function still solve a PDE when $n>1$? We will start with the case $n=2$ in \Cref{sec:case_two} and show that the answer is yes, thereby establishing our main result. We will then extend the argument to all $n>1$ in \Cref{sec:general}.

\section{Main result}\label{sec:case_two}
From now on, we consider the simple case where $p \in [1, 3)$ and $n = 2$. Expanding the step-$2$ log-signature coordinate-wise, the step-$2$ log-ODE method in integral form reads
\begin{equation}\label{logode-2}
    \tensor{Z}_1 = \tensor{Z}_0 + \int_0^1 \tensor{Z}_u \tensorprod \Big(\sum_{i=1}^d \tilde{x}^{(i)}_{s,t} e_i + \sum_{i=1}^d\sum_{j=1}^d \tilde{x}^{(i,j)}_{s,t}e_i\otimes e_j\Big) du.
\end{equation}
where we have written $\log_2S(\tensor{X})_{s,t}=\left(0,(\tilde{x}_{s,t}^{(i)})_{i=1}^{d},(\tilde{x}_{s,t}^{(i,j)})_{i,j=1}^{d}\right)$. In the sequel, to ease notation, we write $x^{(i)}_{s,t}=\tilde{x}^{(i)}_{s,t}$ and $x^{(i,j)}_{s,t}=\tilde{x}^{(i,j)}_{s,t}$. 
Taking the inner product between two such solutions, we now obtain our main result. We defer the proof to \Cref{ssec:main_results_proofs}.
\begin{theorem}\label{thm:pde}
Let $V$ be a vector space of dimension $d$. Consider two $p$-rough paths $\tensor{X},\tensor{Y} \in G\Omega_p(\statespace)$ and denote by $\tensor{Z}^x, \tensor{Z}^y$ the solutions of the step-$2$ log-ODEs (\ref{logode-2}) on $[s,s']$ and $[t,t']$, started at $\tensor{Z}^x_0=\tensora$ and $\tensor{Z}^y_0=\tensorb$ respectively, that is $\tensor{Z}^x_1=\mathcal{L}^2_{s,s'}(\tensora, \tensor{X})$ and $\tensor{Z}^y_1=\mathcal{L}^2_{t,t'}(\tensorb, \tensor{Y})$ respectively. Then, the following real-valued variables
\begin{align}
    f_{u,v}:= \left\langle \tensor{Z}_u^x, \tensor{Z}_v^y \right\rangle, \quad \varphi^{(i)}_{u,v}:= \left\langle \tensor{Z}_u^x \tensorprod e_i, \tensor{Z}_v^y \right\rangle, \quad \psi^{(i)}_{u,v}:= \left\langle \tensor{Z}_u^x, \tensor{Z}_v^y \tensorprod e_i \right\rangle
\end{align}
for $i=1, \ldots, d$, satisfy the following system of PDEs
    \begin{align}\label{eq:pde_log2}
    \frac{\partial^2f}{\partial u \partial v}  &= \gamma f + \sum_{j=1}^d \alpha_j \varphi^{(j)} +  \sum_{j=1}^d \beta_j \psi^{(j)} \\
    \frac{\partial \varphi^{(j)}}{\partial v} &= y^{(j)}_{t,t'} f + \sum_{i=1}^d y^{(i,j)}_{t,t'}\psi^{(i)}\\
    \frac{\partial \psi^{(j)}}{\partial u} &= x^{(j)}_{s,s'} f + \sum_{i=1}^d x^{(i,j)}_{s,s'} \varphi^{(i)}.
\end{align}
for all $(u,v)\in [0,1]\times[0,1]$  with boundary conditions 
\[ \begin{matrix}%
f_{0,v}= \langle \tensora, \tensor{Z}^y_v\rangle,    & \varphi^{(i)}_{0,v}=\left\langle \tensora \tensorprod e_i, \tensor{Z}_v^y \right\rangle, & \psi^{(i)}_{0,v}= \left\langle \tensora, \tensor{Z}_v^y \tensorprod e_i \right\rangle\\ 
f_{u,0}=\langle \tensor{Z}^x_u, \tensorb\rangle,  &  \psi^{(i)}_{u,0}= \left\langle \tensor{Z}_u^x, \tensorb \tensorprod e_i \right\rangle, & \varphi^{(i)}_{u,0}= \left\langle \tensor{Z}_u^x \tensorprod e_i, \tensorb \right\rangle 
\end{matrix}\]
and where
\begin{equation*}
    \alpha_j := \sum_{i=1}^d x^{(i,j)}_{s,s'}y^{(i)}_{t,t'}, \quad \beta_j := \sum_{i=1}^d x^{(i)}_{s,s'}y^{(i,j)}_{t,t'}, \quad 
    \gamma = \sum_{i=1}^{d}x^{(i)}_{s,s'}y^{(i)}_{t,t'}+\sum_{i,j=1}^{d}x^{(i,j)}_{s,s'}y^{(i,j)}_{t,t'}.
\end{equation*}
\end{theorem}
\begin{remark}
    In \Cref{remark:linear_paths}, when we applied the step-$1$ log-ODE method to the signature RDEs and taking the inner product of the solutions, we obtained a PDE \cref{eq:pde_log1} that has the same form as the original Goursat PDE \cite{salvi2021signature}. Here, when applying the step-$2$ method, we have obtained a new system of PDEs. The equation for $f_{u,v}$ has the same form as the original hyperbolic PDE, albeit with additional forcing terms, that involve additional state variables. 
\end{remark}
We emphasize that while the kernel of a pair of bounded variation paths is the solution of the original Goursat PDE, its solution can be approximated by solving \cref{eq:pde_log1_dup} or by solving the bigger system of \Cref{thm:pde}, yielding two schemes with different error rates. Before stating them we establish the uniqueness of the solution of the PDE in \Cref{thm:pde}.
\medskip


Let $\{0=u_0<u_1<...<u_m=1\}$ and $\{0=v_0<v_1<...<v_n=1\}$ be two partitions of the interval $[0,1]$. Define the integral operator 
$$\mathcal{T}: (C([0,1]^2, \mathbb R^{2d+1}), \|\cdot\|_\infty) \to (C([0,1]^2, \mathbb R^{2d+1}), \|\cdot\|_\infty)$$
such that
\begin{align*}
    [\mathcal{T} h(s,t)]_1 &= \int_0^s \int_0^t \!\!\left( \gamma_{u,v} [h(u,v)]_1 + \alpha^\top_{u,v} [h(u,v)]_{2:d+1} +  \beta^\top_{u,v} [h(u,v)]_{d+2:2d+1} \right)\!dudv 
\end{align*}
and 
\begin{align*}
    [\mathcal{T} h(s,t)]_{2:d+1} &= \int_0^t \left([h(s,v)]_{1} y^1_{s,v} + {y^2_{s,v}\!\!}^\top [h(s,v)]_{d+2:2d+1} \right) dv \\
    [\mathcal{T} h(s,t)]_{d+2:2d+1} &=  \int_0^s \left([h(u,t)]_{1} x^1_{u,t} + {x^2_{u,t}\!\!}^\top [h(u,t)]_{2:d+1} \right) du
\end{align*}
where $\gamma \in \mathbb R$, $\alpha, \beta, x^1, y^1 \in \mathbb R^d$ and $ x^2, y^2\in \mathbb R^{d\times d}$ are the piecewise constant coefficients in \Cref{thm:pde}. Then, the system of PDEs in \Cref{thm:pde} is equivalent to
\begin{equation}\label{eqn_phi}
    (\text{id} - \mathcal{T})\Phi = h_0
\end{equation}
where $\Phi = (f, \varphi^1, ..., \varphi^d, \psi^1,...,\psi^d) \in C([0,1]^2,\mathbb R^{2d + 1})$ and $h_0 \equiv (1,a,b)$, where $a,b\in \mathbb R^d$. 
\medskip
We prove the following theorem in \Cref{ssec:main_results_proofs}.
\begin{theorem}\label{thm:uniqueness}
    Assume for all $s,t \in [0,1]$
\[ \begin{matrix}
    \|h_0(s,t)\| \leq h_0, & |\gamma_{s,t}| \leq K_1, & \|\alpha_{s,t}\| \leq K_2, & \|\beta_{s,t}\| \leq K_3 \\
     \|x^1_{s,t}\| \leq K_4, & \|y^1_{s,t}\| \leq K_5, & \|x^2_{s,t}\|_{\text{op}} \leq K_6, & \| y^2_{s,t}\|_{\text{op}} \leq K_7.
\end{matrix}\]%
Then, the system of PDEs in \Cref{thm:pde} admits a unique solution.
\end{theorem}

Although the explicit PDEs for $n>2$ will be derived in the next section, we can already assert the following global error estimate (proved in \Cref{ssec:main_results_proofs}), as the proof  relies on \Cref{lemma:global_error} without invoking any of the PDEs derived later. 
\begin{theorem}\label{thm:global_error_second_order_pde}
Let $p \geq 1$ be a real number and $n\geq \lfloor p\rfloor$ be an integer. Consider two $p$-rough paths $\tensor{X},\tensor{Y} \in  G\Omega_p(\statespace)$ controlled by $\omega_x$ and $\omega_y$ respectively. Let $N$ and $N'$ be two integers such that $(ew_x(0,T)/N)^{1/p}<1$ and $(ew_y(0,T)/N')^{1/p}<1$. 
Consider a partition $\{s_i\}_{i=1}^{N}$ of $[0,T]$ such that $w_x(s_i,s_{i+1})\leq w_x(0,T)/N$. Similarly, consider a partition $\{t_i\}_{i=1}^{N'}$ of $[0,T]$ such that $w_y(t_i,t_{i+1})\leq w_y(0,T)/N'$. Then, there exists a constant $C = C(\omega_x(0,T), \omega_y(0,T),p)$ such that   
\begin{align}
    \left|\langle S(\tensor{X}), S(\tensor{Y}) \rangle - f(T, T)\right| \leq C\max\bigg\{\left(\frac{e\omega_x(0,T)}{N}\right)^{\frac{n+1}{p}-1}\!, \left(\frac{e\omega_y(0,T)}{N'}\right)^{\frac{n+1}{p}-1}\bigg\}
\end{align}
\end{theorem}




    

    

\section{The general case}\label{sec:general}






To formulate the result in the general case $n\geq 1$, we first recall the definitions of the adjoints of the linear maps given by the left and right tensor multiplication in the tensor algebra by a tensor $\tensora\in T((V))$. There is a natural pairing $[\cdot, \cdot]$ between the infinite formal series $T((\statespace))$ and the finite sequences in $T(\statespace^*)$, given for all $\tensora\in T((\statespace))$ and all $\tensorb\in T(\statespace^*)$ by 
$$[\tensora, \tensorb]=\left[(a^0, a^1, a^2, \ldots), (b^0, b^1, \ldots, b^m)\right]=\sum_{k=0}^{m}[a^k, b^k]_k$$
where $a^k\in \statespace^{\otimes k}$ and $b^k\in (\statespace^*)^{\otimes k}$ and the pairing $[\cdot,\cdot]_k$ is the canonical one induced by the pairing between $\statespace$ and $\statespace^*$. We will say that a pairing $[\cdot,\cdot]$ between $E$ and $F$ is non-degenerate if $[e,f]=0$ for all $e$ implies $f=0$. The pairing between $T((\statespace))$ and $T(\statespace^*)$ is non-degenerate in both directions.

\begin{definition}
    Let $E,F$ be vector spaces with a pairing $[e,f]$ for $e\in E$ and $f\in F$. Suppose $T$ is a linear operator from $E$ to $E$, and $T^*$ is a linear operator from $F$ to $F$. Then, we say that $T^*$ is adjoint to $T$ if 
$$[e, T^*(f)]=[T(e),f]$$
for all $e\in E$ and $f\in F$. 
\end{definition} 
Provided the pairing between $E$ and $F$ is non-degenerate, the adjoint $T^*$ of $T$ is always unique if it exists.

\subsection{Adjoint of tensor multiplication}

Initially, let $a^k=e_{i_1}\otimes \ldots \otimes e_{i_k}$ in $\statespace^{\otimes k}$. We can write a basis of $T(\statespace^*)$ in terms of words $e^*_{i_1}\otimes\ldots\otimes e^*_{i_m}$. We define a function $\ladjmap{a^k}$ on this basis as follows. Given a word $u=e^*_{i_1}\otimes\ldots\otimes e^*_{i_k}\otimes \tau$ that begins with the sequence $e^*_{i_1}\otimes\ldots\otimes e^*_{i_k}$, the function $\ladjmap{a^k}$ takes the value $\tau$. 
Given a basis element $v$ whose initial sequence is not $e^*_{i_1}\otimes\ldots\otimes e^*_{i_k}$, the function $\ladjmap{a^k}$ maps the word $v$ to zero, that is,  
\begin{equation}
    \ladjmap{a^k}(e^*_{j_1}\otimes\ldots \otimes e^*_{j_n})=\left\{
  \begin{array}{ll}
    e^*_{j_{n-k}}\otimes\ldots\otimes e^{*}_{j_n} & \text{if } n\geq k,~ j_1=i_1, \ldots, j_k=i_k\\
    0,  & \mbox{otherwise}.
  \end{array}
\right.
\end{equation}
A function on a vector space defined on a basis uniquely extends to a linear map on the space. The map $\ladjmap{\tensora}$ from $T(\statespace^*)$ to itself is the adjoint of the linear map $L_\tensora:\tensor{X}\mapsto \tensora\tensorprod\tensor{X}$ from  $(T((V)),\tensorprod)$ to itself. 

The function is zero for any element $\tensorb=(b^0, b^1, \ldots, b^m)$ of $T(\statespace^*)$ if $k>m$. Hence, only finitely of the $e_{i_1}\otimes\ldots\otimes e_{i_k}$ contribute, allowing us to extend the linear map unambiguously to a linear map of  $T((\statespace))$ to linear maps from $T(\statespace^*)$ to $T(\statespace^*)$.

\begin{remark}
    Sometimes it is convenient to identify words of $T(\statespace^*)$ with words of $T((\statespace))$, and consider the map $\ladjmap{\tensora}$ as running from $T((\statespace))$ to $T((\statespace))$. We note that this map is well defined from signatures to $T((\statespace))$, because the coefficients on the right handside converge by the factorial decay of the signature. See \Cref{ssec:adjoints_proofs} for a proof.
\end{remark}

Similarly, we define the adjoint of the linear map $R_{\tensora}:\tensor{X}\mapsto \tensor{X}\tensorprod \tensora$ from $(T((\statespace)), \tensorprod)$ to itself, by initially considering $a^k=e_{i_1}\otimes \ldots \otimes e_{i_k}$ in $\statespace^{\otimes k}$ and the map $\radjmap{a^k}$ that given a word $u=\tau\otimes e^{i_1}\otimes\ldots\otimes e^{i_k}$ that ends with the sequence $e^*_{i_1}\otimes\ldots\otimes e^*_{i_k}$, the function $\radjmap{a^k}$ takes the value $\tau$. Given a basis element $v$ whose last sequence is not $e^*_{i_1}\otimes\ldots\otimes e^*_{i_k}$, the function $\radjmap{a^k}$ maps the word $v$ to zero. 

\begin{equation}
    \radjmap{a^k}(e^*_{j_1}\otimes\ldots \otimes e^*_{j_n})=\left\{
  \begin{array}{ll}
    e^*_{j_{1}}\otimes\ldots\otimes e^{*}_{j_{n-k-1}} & \text{if } n\geq k,~ j_{n-k}=i_1, \ldots, j_n=i_k\\
    0,  & \mbox{otherwise}.
  \end{array}
\right.
\end{equation}

\subsection{Rewriting the PDE from \Cref{thm:pde}} With this, we see that the variables appearing in the PDE system obtained in \Cref{thm:pde}, governing the inner product of the solutions of two step-$2$ log-ODE solutions $\tensor{Z}^x_u$ and $\tensor{Z}^y_v$, correspond to coordinates of $\ladj{\tensor{Z}_u^x}{\tensor{Z}_v^y}$ and $\ladj{\tensor{Z}_v^y}{\tensor{Z}_u^x}$, namely $$\varphi^{(i)}_{u,v}:= \big\langle \ladj{\tensor{Z}_u^x}{\tensor{Z}_v^y}, e_i \big\rangle, \quad \psi^{(i)}_{u,v}:= \big\langle\ladj{\tensor{Z}_v^y}{\tensor{Z}_u^x},  e_i\big\rangle.$$
Therefore, we can restate the result as follows. The variables
\begin{align}
    f_{u,v}&:= \left\langle \tensor{Z}_u^x, \tensor{Z}_v^y \right\rangle, \quad
    \varphi_{u,v}:= \pi_1(\ladj{\tensor{Z}_u^x}{\tensor{Z}_v^y}), \quad \psi_{u,v}:= \pi_1(\ladj{\tensor{Z}_v^y}{\tensor{Z}_u^x})
\end{align}
satisfy the following system of PDEs
    \begin{align}
    \frac{\partial^2f}{\partial u \partial v}  &= \gamma f + \langle \alpha, \varphi\rangle_1 + \langle \beta, \psi\rangle_1 \\
    \frac{\partial \varphi}{\partial v} &= y^1_{t,t'} f +\ladj{\psi}{y^{2}_{t,t'}}\\
    \frac{\partial \psi}{\partial u} &= x^1_{s,s'} f+  \ladj{\varphi}{x^{2}_{s,s'}} .
\end{align}
with appropriate boundary conditions
and coefficients $\alpha, \beta$ and $\gamma$ given by
\begin{align*}
    \alpha := \radj{y^{1}_{t,t'}}{x^{2}_{s,s'}}, \quad \beta := \radj{x^{1}_{s,s'}}{y^{2}_{t,t'}}, \quad 
    \gamma = \sum_{k=1}^{n}\left\langle x^{k}_{s,s'}, y^{k}_{t,t'}\right\rangle_k.
\end{align*}

\subsection{The step-$n$ log-ODE method for the signature kernel}\label{ssec:new_pde}

The following theorem (proved in \Cref{sec:generic_n}) generalises our main result beyond the case $n=2$. 
\begin{theorem}\label{thm:pde_genericn}
Consider two $p$-rough paths $\tensor{X},\tensor{Y} \in  G\Omega_p(\statespace)$. Let $n\geq \lfloor p \rfloor$ and denote by $\tensor{Z}^x, \tensor{Z}^{y}$ the solutions of the step-$n$ log-ODEs (\ref{logode-2}) started at $\tensor{Z}^x_0=\tensora,\tensor{Z}^y_0=\tensorb$. Then, the scalar variables
\begin{align}
    f_{u,v}&:= \left\langle \tensor{Z}_u^x, \tensor{Z}_v^y \right\rangle, \quad \varphi^0_{u,v}:= 0, \quad \psi^0_{u,v}:= 0
\end{align}
and the $\statespace^{\otimes k}$-valued variables 
\begin{align}
    \varphi^k_{u,v}&:= \pi_k(\ladj{\tensor{Z}_u^x}{\tensor{Z}_v^y}), \quad \psi^k_{u,v}:= \pi_k(\ladj{\tensor{Z}_v^y}{\tensor{Z}_u^x}) \quad \text{for }k=1, \ldots, n-1
\end{align}
satisfy the following system of PDEs
    \begin{align}
    \frac{\partial^2f}{\partial u \partial v}  &= \gamma f + \sum_{k=0}^{n-1}\langle \alpha^{k}, \varphi^{k}\rangle_k +  \sum_{k=0}^{n-1}\langle \beta^{k}, \psi^{k}\rangle_k \\
    \frac{\partial \varphi^k}{\partial v} &= y^k_{t,t'} f + \sum_{p=0}^{k}\varphi^{p}\otimes y^{k-p}_{t,t'}+\sum_{p=0}^{n-k}\ladj{\psi^{p}}{y^{k+p}_{t,t'}}-\sum_{p=0}^{k}\langle y^p_{t,t'},\psi^p\rangle_p\delta_{k,0} \\
    \frac{\partial \psi^k}{\partial u} &= x^k_{s,s'} f+ \sum_{p=0}^{k}\psi^{p}\otimes x^{k-p}_{s,s'}+ \sum_{p=0}^{n-k}\ladj{\varphi^{p}}{x^{k+p}_{s,s'}} - \sum_{p=0}^{k}\langle x^p_{s,s'},\phi^p\rangle_p\delta_{k,0}.
\end{align}
for all $(u,v)\in[0,1]\times[0,1]$ with boundary conditions
\[ \begin{matrix}%
f(0,v)= \langle \tensora, \tensor{Z}^y_v\rangle,    & \psi^k(0,v)=\pi_k(\ladj{\tensor{Z}^y_v}{\tensora}),& \varphi^k(0,v)= \pi_k(\ladj{\tensora}{\tensor{Z}^y_v}),\\
f(u,0)=\langle \tensor{Z}^x_u, \tensorb\rangle,  &  \varphi^k(u,0)= \pi_k(\ladj{\tensor{Z}^x_u}{\tensorb}), & \psi^k(u,0)= \pi_k(\ladj{\tensorb}{\tensor{Z}^x_u}).
\end{matrix}\]%
and where
\begin{equation*}
    \alpha^k := \sum_{p=1}^{n-k}\radj{y^{p}_{t,t'}}{x^{p+k}_{s,s'}}, \quad \beta^k := \sum_{p=1}^{n-k} \radj{x^p_{s,s'}}{y^{p+k}_{t,t'}}, \quad 
    \gamma = \sum_{k=1}^{n}\left\langle x^k_{s,s'}, y^{k}_{t,t'}\right\rangle_k.
\end{equation*}
\end{theorem}

Since one can take any $n\geq \lfloor p \rfloor$, and since the two rough paths could a have different parameter $p$, there is no reason why one could not approximate the signature of $\mathbb X$ using a step-$n$ log-ODE and the signature of $\mathbb Y$ using a step-$n'$ log-ODE with $n'\neq n$. It is easy to generalise the result above to this setting. 

The final step to produce a numerical scheme for approximating the signature kernel is to discretise the PDEs. To maintain the rate of convergence, this discretisation must yield an accurate enough approximation of the solution. A discretization method is given in \Cref{alg:full_algo}, where for simplicity, we take a single step for each PDE. We leave as future work the verification that the chosen discretisation ensures that the convergence rate remains optimal.

  \begin{algorithm}[h]
    \caption{PDE discretization method}
    \label{alg:full_algo}
        \begingroup
\setstretch{1.3} 
    \begin{algorithmic}[1]
    \vspace{3pt}
    \State \textbf{Input:} Truncated log-signatures over a partition $\logsig{x}{i}\in \mathcal{L}^n(\mathbb R^d)$ for $i=1,\ldots,N$ and $\logsig{y}{j}\in \mathcal{L}^n(\mathbb R^d)$ for $j=1,\ldots,N'$.  Initial conditions $f_{0,\cdot}, f_{\cdot,0}, \varphi_{0,\cdot}, \varphi_{\cdot,0}, \psi_{0,\cdot}, \psi_{\cdot,0}$. 
    \For{$i$ in $0,\ldots,N-1$} 
     \For{$j$ in $0,\ldots,N'-1$} 
     \State $\widetilde{\logsig{x}{i}}, \widetilde{\logsig{y}{j}}\leftarrow \pi_{n-1}(\logsig{x}{i}), \pi_{n-1}(\logsig{y}{j})$
     \State \textit{// Update the adjoint states}
   \State $\varphi_{i+1,j+1}\leftarrow \varphi_{i,j+1} + f_{i,j} \widetilde{\logsig{x}{i}} + \varphi_{i,j+1}\tensorprod\widetilde{\logsig{x}{i}}+\ladj{\psi_{i,j+1}}{\logsig{x}{i}} - \langle\psi_{i,j+1}, \widetilde{\logsig{x}{i}}\rangle $
   \State $\psi_{i+1,j+1}\leftarrow \psi_{i+1,j} + f_{i,j} \widetilde{\logsig{y}{j}} + \psi_{i+1,j}\tensorprod\widetilde{\logsig{y}{j}} + \ladj{\varphi_{i+1,j}}{\logsig{y}{j}} - \langle\varphi_{i+1,j}, \widetilde{\logsig{y}{j}}\rangle $
  \State \textit{// Intermediate states}
    \State $h_1\leftarrow f_{i,j} \langle \logsig{x}{i}, \logsig{y}{j}\rangle +  \langle \varphi_{i,j}, \radj{\logsig{x}{i}}{\logsig{y}{j}}\rangle + \langle \psi_{i,j}, \radj{\logsig{y}{j}}{\logsig{x}{i}}\rangle$
   \State $h_2\leftarrow f_{i,j+1} \langle \logsig{x}{i}, \logsig{y}{j}\rangle +  \langle \varphi_{i,j+1}, \radj{\logsig{x}{i}}{\logsig{y}{j}}\rangle + \langle \psi_{i,j+1}, \radj{\logsig{y}{j}}{\logsig{x}{i}}\rangle$
  \State $h_3\leftarrow f_{i+1,j} \langle \logsig{x}{i}, \logsig{y}{j}\rangle +  \langle \varphi_{i+1,j}, \radj{\logsig{x}{i}}{\logsig{y}{j}}\rangle + \langle \psi_{i+1,j}, \radj{\logsig{y}{j}}{\logsig{x}{i}}\rangle$
\State $f^p\leftarrow f_{i+1,j} + f_{i,j+1} - f_{i,j} + h_1$
 \State $h_4\leftarrow  f^p  \langle \logsig{x}{i}, \logsig{y}{j}\rangle +  \langle \varphi_{i+1,j+1}, \radj{\logsig{x}{i}}{\logsig{y}{j}}\rangle + \langle \psi_{i+1,j+1}, \radj{\logsig{y}{j}}{\logsig{x}{i}}\rangle$
  \State \textit{// Update the kernel state}
    \State  $f_{i+1,j+1} =  f_{i+1,j} + f_{i,j+1} - f_{i,j} + (1./4)(h_1 + h_2 + h_3 + h_4)$ 
    \EndFor 
    \EndFor    
    \State\Return $f_{N,N'}$. 
    \end{algorithmic}
    \endgroup
\end{algorithm}

If the data takes the form of time series $x_{1},\ldots,x_{m}$ with $x_{i}\in \mathbb R^d$, one first needs to construct the higher order description. In this case, the input of \Cref{alg:full_algo} might be obtained by embedding the data into a continuous path $x:[0,T]\to \mathbb R^d$ and then computing the log-signatures $\log_n S(x)_{s_i,s_{i+1}}$ of the path on
each segment of the partition $\{0=s_0< \ldots < s_{N}=T\}$. Such constructions are straightforward using Python packages such as \texttt{esig}, \texttt{iisignature}, \texttt{signatory}, \texttt{signax} \cite{esig,reizenstein2018iisignature,kidger2020signatory,signax} or \texttt{RoughPy} \cite{roughpy}. It is worth noting that for small degrees $n=1$ and $n=2$, the equations can still be described in terms of matrix-vector and dot products, and the numerical schemes can be implemented using native Python and \texttt{NumPy} functions. However, packages designed for working with fundamental objects from free non-commutative algebra (see lines $6-14$ in \Cref{alg:full_algo}) such as \texttt{RoughPy}, significantly streamline the implementation and offer a seamless transition between different schemes. An implementation of the newly proposed methods for computation of signature kernels is made accessible at \url{https://github.com/maudl3116/high-order-sigkernel}.

\begin{remark}\label{remark:PAB}
It can be shown \cite[Theorem 2.]{friz2009rough} that the step-$n$ log-ODE method approximation coincides with the solution of a RDE driven by a \emph{piecewise log-linear path} of degree $n$ w.r.t a partition $\part=\{t_i\}_{i=0}^{N}$ of $[0,T]$. Considering the signature RDE (\ref{eq:rde}) driven by the $p$-rough path $\mathbb X$, its solution is approximated by replacing $\mathbb X$ by the path $\mf{X}{}{n,\part}$ of Lie degree $n$ (with $n\geq \lfloor p \rfloor$) defined by $\mathbb{X}^{n,\part}_0=\mathbb{X}_0$ and for any $t_i\leq s<t\leq t_{i+1}$ by
\begin{align}
    \mf{X}{s,t}{n,\part}=
    \exp_n\left(\frac{t-s}{t_{i+1}-t_i}\log_nS(\mathbb X)_{t_i, t_{i+1}}\right)
\end{align}
and extended to be multiplicative on $s,t\in[0,T]$. 
This generalises classical piecewise linear approximations. We refer to these paths as \emph{piecewise log-linear path} \cite{bellingeri2022smooth}, and note that in the literature,  they are also referred to as “piecewise abelian paths” in \cite{flint2015pathwise} and as “pure rough paths” in \cite{boedihardjo2020path}. Therefore, the state variables in the PDE system \Cref{thm:pde_genericn} can be re-expressed in terms of the signatures of piecewise log-linear approximations. In particular $f_{u,v}=\langle \mathbb Z^x_u, \mathbb Z^y_v\rangle$ is the signature kernel of $\mathbb{X}^{n,D}$, $\mathbb{Y}^{n,D'}$, i.e. $f_{u,v}=\langle S(\mathbb{X}^{n,D}),  S(\mathbb{Y}^{n,D'})\rangle$. 
\end{remark}

\subsection{The signature kernel of smooth rough paths}
So far, we have derived new numerical schemes for approximating the signature kernel of $p$-rough paths, that involve solving systems of PDEs. We saw in \Cref{remark:PAB} that these log-PDE schemes can be obtained by replacing the driving signals by piecewise log-linear paths. In this section, we show that the signature kernel of smooth rough paths, which form a dense subset of the space of $p$-rough paths \cite{bellingeri2022smooth}, solves an analogous PDE system whose coefficients are smooth functions given by higher order iterated integrals of the driving signals. Equivalently, one may derive these PDEs directly in the smooth rough path setting and then specialise them to the piecewise log-linear approximations, recovering \Cref{thm:pde} and \Cref{thm:pde_genericn} as corollaries.

\medskip 
A level-$n$ smooth geometric rough path ($n$-sgrp)
is a path $\mathbb X: [0,T] \to T^n(\statespace)$ such that the shuffle relation $\langle \mathbb X_t, v\shuffle w\rangle=\langle \mathbb X_t, v\rangle\langle \mathbb X_t, w\rangle $ holds for all times $t\in [0,T]$ and all words $|v|+|w|\leq n$, and 
for any word $w$ of length $|w|\leq n$, the map $t\mapsto \mathbb \langle \mathbb{X}_t,w\rangle$ is smooth. We refer the reader to \cite{bellingeri2022smooth} for the construction and theory of smooth rough paths, including the definition of their signatures. The signature kernel $\langle S(\mf{X}{}{}),S(\mf{Y}{}{})\rangle$ of two level-$n$ smooth rough paths $\mf{X}{}{}$ and $\mf{Y}{}{}$ solves an augmented system of PDEs which generalises the original Goursat problem (retrieved when $n=1$). 
Let $\mf{X}{}{}$ and $\mf{Y}{}{}$ be two level-$n$ smooth rough paths with diagonal derivatives
\begin{align*}
\dot{\srp{x}}_s:=\partial_v|_{v=s}\mf{X}{s,v}{}\in \mathcal{L}^n(\statespace), \quad \dot{\srp{y}}_t:=\partial_v|_{v=t}\mf{Y}{t,v}{}\in\mathcal{L}^n(\statespace)
\end{align*} 
in the space of Lie polynomials of degree $n$. 
The real-valued function indexed on the plane $f:I\times J \to \mathbb R$ and the two tensor-valued functions $\adja: I\times J \to \pi_n(T^{>0}(\statespace)$ and $\adjb: I\times J \to \pi_n(T^{>0}(\statespace)$ defined for all $(s,t)\in I\times J$ by 
    \begin{align*}
    f(s,t)&:=\langle S(\mf{X}{}{})_{0,s},S(\mf{Y}{}{})_{0,t}\rangle \\ 
    \adja(s,t)&:=\pi_n(\ladj{S(\mf{Y}{}{})_{0,t}}{S(\mf{X}{}{})_{0,s}}-f(s,t)\cdot\identity)\\ 
    \adjb(s,t)&:=\pi_n(\ladj{S(\mf{X}{}{})_{0,s}}{S(\mf{Y}{}{})_{0,t}}-f(s,t) \cdot\identity)
    \end{align*}
solve the following system of linear PDE
    \begin{align}
\mixed{f}&= \langle \dot{\srp{x}}_s,\dot{\srp{y}}_t\rangle f  + \langle \adja, \radj{\dot{\srp{x}}_s}{\dot{\srp{y}}_t} \rangle + \langle \adjb, \radj{\dot{\srp{y}}_t}{\dot{\srp{x}}_s} \rangle \label{eq:sgrp_ker}\\ 
\single{s}{\adja}&= \dot{\srp{x}}_sf+\adja\otimes_n \dot{\srp{x}}_s+ \ladj{\adjb}{\dot{\srp{x}}_s} -\langle \adjb, \dot{\srp{x}}_s\rangle\identity_n\label{eq:sgrp_phi}\\
\single{t}{\adjb} &=   \dot{\srp{y}}_tf+\adjb\otimes_n \dot{\srp{y}}_t+\ladj{\adja}{\dot{\srp{y}}_t}-\langle \adja, \dot{\srp{y}}_t\rangle\identity\label{eq:sgrp_psi}
\end{align}
with boundary conditions 
\[ \begin{matrix}
f(0,t)= 1,    & \adjb(0,t)=\mf{Y}{0,t}{}-\identity_n,& \adja(0,t)= 0,\\
f(s,0)=1,  &  \adja(s,0)= \mf{X}{0,s}{}-\identity_n, & \adjb(s,0)= 0.
\end{matrix}\]%

\begin{remark}
    In \cref{eq:sgrp_phi} and \cref{eq:sgrp_psi}, $\otimes_n$ and $\identity_n$ denote the tensor product and the identity in $T^n(\statespace)$ respectively. However, since the truncated adjoint variables $\adja, \adjb$ and the kernel $f$ are the solution of the system \cref{eq:sgrp_ker}, \cref{eq:sgrp_phi} and \cref{eq:sgrp_psi}, for each $n\geq 0$, we can formally write the differential equation for the full adjoints as functions from $I\times J$ to $T^{>0}(V)$, only replacing $\otimes_n$ with $\tensorprod$ and $\identity_n$ with $\identity$ in the system. 
\end{remark}

\begin{remark} Note that the $n^{th}$ degree of $\radj{\dot{\srp{x}}_s}{\dot{\srp{y}}_t}$ and $\radj{\dot{\srp{y}}_t}{\dot{\srp{x}}_s}$ are both zero. Therefore, the inner products in the PDE \cref{eq:sgrp_ker} can be rewritten as 
\begin{align*}
    \langle \adja, \radj{\dot{\srp{x}}_s}{\dot{\srp{y}}_t}\rangle = \langle \pi_{n-1}(\adja), \pi_{n-1}(\radj{\dot{\srp{x}}_s}{\dot{\srp{y}}_t})\rangle\\
    \langle \adjb, \radj{\dot{\srp{y}}_t}{\dot{\srp{x}}_s}\rangle = \langle \pi_{n-1}(\adjb), \pi_{n-1}(\radj{\dot{\srp{y}}_t}{\dot{\srp{x}}_s})\rangle,
\end{align*}
and only depend on the degrees of the adjoint variables that are lower than $n-1$. This is in line with the previous results on signature kernels, since for two $1$-smooth rough paths, with diagonal derivatives $\dot{\srp{x}}_s=(0,\dot{x}_s)$ and $\dot{\srp{y}}_t=(0,\dot{y}_t)$, the reader can check that the system, written component-wise, reads as
\begin{align}
\mixed{f}&= \langle \dot{x}_s,\dot{y}_t\rangle_1 f\label{eq:kertrunc1_compo}\\ 
  \single{s}{\adja^{(k)}}&=\dot{x}^{(k)}_s f \qquad\qquad \text{for all~} k\in\{1,\ldots, d\}\\ 
\single{t}{\adjb^{(k)}}&=\dot{y}^{(k)}_t f\qquad\qquad \text{for all~} k\in\{1,\ldots, d\}.
\end{align}
Since the states $\adja^{(k)}$ and $\adjb^{(k)}$ do not influence the dynamics of $f$, to compute the signature kernel, we can solve \cref{eq:kertrunc1_compo} which corresponds to solving the original Goursat PDE problem \cite{salvi2021signature}. Unless both $\mf{X}{}{}$ and $\mf{Y}{}{}$ are minimal extension of $1$-smooth rough paths, the dynamics of the kernel are influenced by the first degrees of extra state defined in terms of $\ladj{S(\mf{Y}{}{})}{S(\mf{X}{}{})}$ and $\ladj{S(\mf{X}{}{})}{S(\mf{Y}{}{})}$. 
As aforementioned, the dynamics of $f$ only depend on the degrees $k=1,\ldots,n$ of $\adja$ and $\adjb$ because the action of the adjoints $\radjmap{\dot{\bm{x}}_s}$ and $\radjmap{\dot{\bm{y}}_t}$ can only decrease the maximum degree (in fact there are also independent from the degree $n$ of $\adja$ and $\adjb$, because the degree $n$ of $\radj{\dot{\bm{x}}_s}{\dot{\bm{y}}_t}$ is given by $\radj{\dot{x}^{0}_s}{\dot{y}^{n}_t}=\radj{0}{\dot{y}^{n}_t}=0$). Futhermore, the system $(f,\adja^{n},\adjb^{n})$ is closed because the dynamics of $\adja^{n}$ and $\adjb^{n}$ do not depend on degrees strictly higher than $n$ of $\adja$ and $\adjb$.
\end{remark}


\bigskip 
\begin{remark}
Suppose $\mf{X}{}{}$ and $\mf{Y}{}{}$ are two $m$-smooth geometric rough paths. Their diagonal derivatives can be embedded in $\mathcal{L}^{n}(\statespace)$ for some $n>m$. 
However, for $\dot{\bm{x}}_s$ and $\dot{\bm{y}}_t$, all degrees greater than $m+1$ are zero. Denote the solution of this system by $h^n:=(f,\adja^n, \adjb^n)$ and consider $h^{\text{proj}}=(f,\pi_{\leq m}(\adja^n), \pi_{\leq m}(\adjb^n))$. We have $h^{\text{proj}}=h^{m}$. In other words, when minimally extending the input smooth rough paths to higher degrees, we write a bigger set of equations, but if we solve them and project the solution, we get the same state as if we had solved the reduced set of equations, which shows consistency. To fix ideas, let's continue the example in the previous remark. If we extend the $1$-sgrp to $2$-sgrp, then we get  
\begin{align}
   \mixed{f}&=  \langle \dot{x}_s, \dot{y}_t \rangle_1 f\\ 
   \single{s}{\adja^{(k)}}&=\dot{x}^{(k)}_s f \qquad\qquad \text{for all~} k\in\{1,\ldots, d\}\\
    \single{t}{\adjb^{(k)}}&=\dot{y}^{(k)}_t f\qquad\qquad \text{for all~} k\in\{1,\ldots, d\}\\
   \single{s}{\adja^{(k,p)}}&= \dot{x}^{(p)}_s \adja^{(k)}\qquad\qquad \text{for all~} k,p\in\{1,\ldots, d\}\\
   \single{t}{\adjb^{(k,p)}}&= \dot{y}^{(p)}_t \adjb^{(k)}\qquad\qquad \text{for all~} k,p\in\{1,\ldots, d\}
\end{align}
For $1$-sgrps, or any minimal extensions of $1$-sgrps, the dynamics of $f$ are independent from any other state. Otherwise, we need more equations to get $f$.
\end{remark}



\section{Numerical illustration}\label{sec:experiments}
Consider two multivariate Brownian motion sample paths $\bmx,\bmy:[0,1]\to \mathbb R^d$. Almost surely $\bmx$ and $\bmy$ have infinite
two-variation and finite $p$-variation for every $p > 2$. Therefore, it is not possible to define their signature kernel as the solution of the Goursat PDE in \cite{salvi2021signature}. However, any two piecewise linear approximations of $\bmx$ and $\bmy$ associated with any two partitions $D,D'$ of $[0,1]$ satisfy a PDE of the form \cref{eq:pde_log1}, where the single coefficient is determined by the increments $\{W_{s_i,s_{i+1}}\}_{i=0}^{N-1}$ and $\{V_{t_j,t_{j+1}}\}_{j=0}^{N'-1}$ of the paths $W$ and $V$ on their respective partitions $\{s_i\}_{i=0}^{N}$ and $\{t_j\}_{j=0}^{N'}$. The solution converges to the signature kernel of $\bmx$ and $\bmy$ as the mesh size of the partitions tends to $0$, since the signatures of piecewise linear approximations converge to the Stratonovich signatures of $\bmx$ and $\bmy$ \cite[Sec. 3.3.2]{lyons2004differential}. 
To increase the order of approximation, we add higher order terms on top of the increment. For example, when $n=2$, we use the terms given by the L\'{e}vy area for all $p,q= 1,\ldots,d$ by  
\begin{align*}
    \areax^{(p,q)}_{s,s'} = \frac{1}{2}{\int_{\sigma'=s}^{s'}\int_{\sigma=s}^{\sigma'}}\circ\!dW^{(p)}(\sigma)\circ \!dW^{(q)}(\sigma') - \circ dW^{(q)}(\sigma)\circ \!dW^{(p)}(\sigma').
\end{align*}
To compare the newly derived log-PDE schemes in this article, we use the following estimator of the error
\begin{figure}[h]
\centering
\includegraphics[scale=0.8]{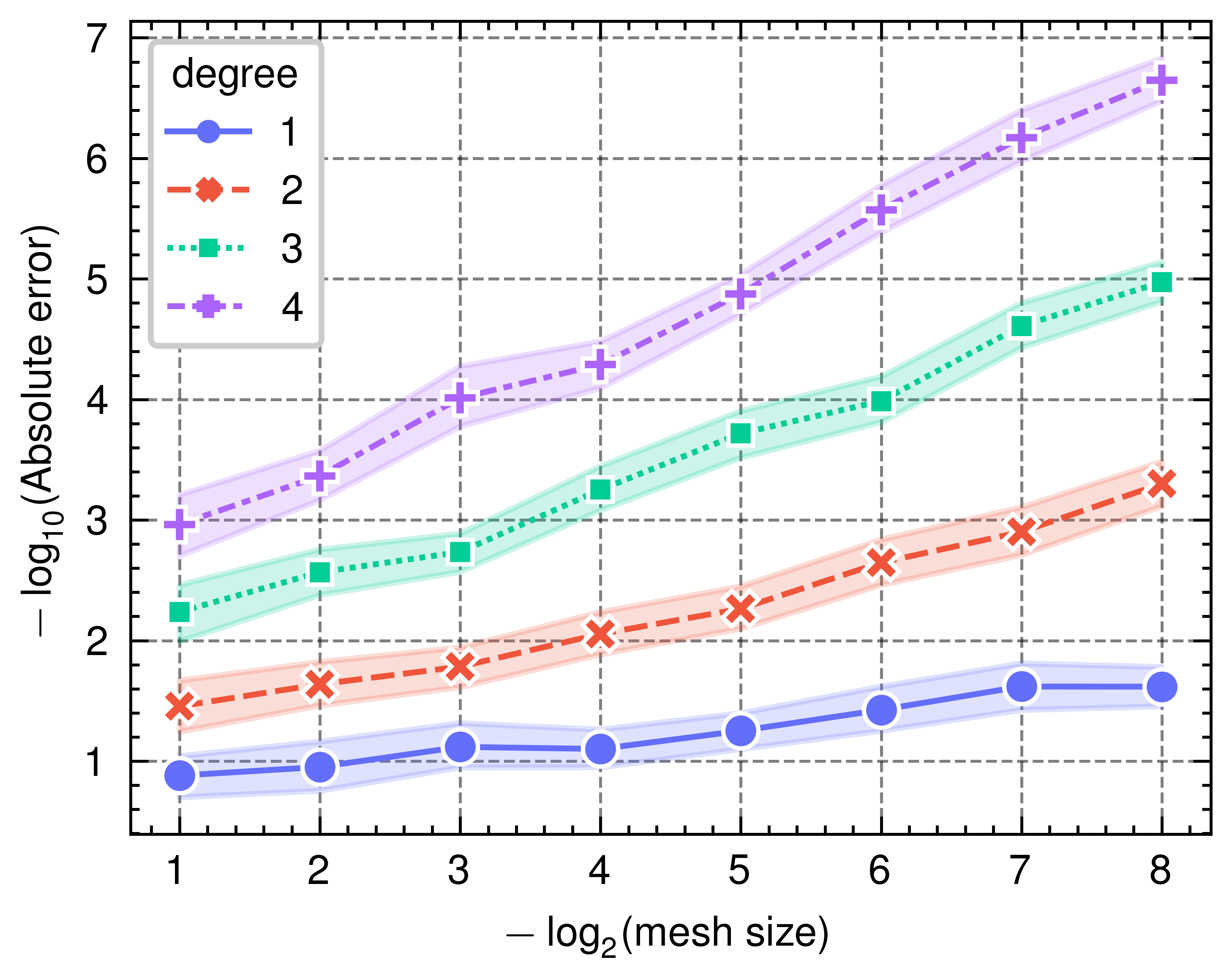}
\caption{Log-PDE approximation of the signature kernel of two BM sample paths. }\label{fig:rate_exp}
\end{figure}
\begin{align}\label{eq:error_estimator}
    \hat{E} = |\hat{f}^{\text{fine}}(1,1)-\hat{h}^0(1,1)|
\end{align}
where $\hat{f}^{\text{fine}}$ denotes the numerical solution (obtained with \Cref{alg:full_algo}) of the one-dimensional Goursat problem obtained by approximating the inputs by \acrlong{PABs} of degree $1$ on a regular partition of small mesh size $\Delta t= 1/N$; and $\hat{h}_0$ is the first coordinate of numerical solution of the PDE system obtained by approximating the inputs by \acrlong{PABs} of degree $n$ on a regular partition of mesh size $\Delta t= 2^k/N$. For the experiment, we consider $2$-dimensional Brownian motion, and construct the fine grid with $N=2^{12}$ subintervals. \Cref{fig:rate_exp} shows the error as a function of the mesh size coarsening factor $k\in\{4, 5, \ldots, 11\}$ for different degrees $n\in\{1,2,3,4\}$. The experiment is repeated on $100$ pairs of sample paths from $\mathbb R^2$-valued Brownian motion on the unit time interval $[0,1]$, and the dots on \Cref{fig:rate_exp} correspond to the mean of the errors calculated according to \cref{eq:error_estimator}. We see that, for any fixed partitioning of the interval $[0,1]$, the error decreases with the degree $n$. As expected, for any degree $n$, the error also decreases with the mesh size.


\section{Conclusion and future directions}\label{sec:conclusion}

In this paper, we have addressed the computational intractability of existing PDE solvers for signature kernels when applied to highly oscillatory paths. Building on the analysis of \cite{salvi2021signature}, we extended the theory to encompass rough paths by leveraging the framework of smooth rough paths. This led to a new system of PDEs, whose coefficients are expressed in terms of higher-order iterated integrals of the input rough paths, to approximate rough signature kernels. We proved the well-posedness of this system and derived quantitative error bounds, resulting in a higher-order and tractable approximation to rough signature kernels.
Possible extensions of this work include developing similar schemes for so-called \emph{weighted signature kernels} \cite{cass2024weighted} and developing adaptive versions, where the partition size and the degree of the \acrlong{PABs} are adjusted instead of being fixed at the beginning, possibly building on \cite{bayer2023adaptive}. Assuming that smooth rough paths are limit points bounded variation paths, our system of PDEs \cref{eq:sgrp_ker}, \cref{eq:sgrp_phi}, \cref{eq:sgrp_psi}, can be seen as an homogenisation limit of a sequence of Goursat PDEs from \cite{salvi2021signature}. We leave this homogenisation interpretation for future work.

\section*{Acknowledgements}
This work was supported in part by EPSRC (NSFC) under Grant EP/S026347/1, in part by The Alan Turing Institute under the EPSRC grant EP/N510129/1, the Data Centric Engineering Programme (under the Lloyd’s Register Foundation grant G0095), the Defence and Security Programme (funded by the UK Government) and the Office for National Statistics \& The Alan Turing Institute (strategic partnership) and in part by the Hong Kong Innovation and Technology Commission (InnoHK Project CIMDA).

\section*{Declarations}
For the purpose of open access, the authors have applied a Creative Commons Attribution (CC BY) license to any Accepted Manuscript version arising.\\

\bibliographystyle{siamplain}
\bibliography{references}

\begin{thebibliography}{10}

\bibitem{arribas2020sigsdes}
{\sc I.~P. Arribas, C.~Salvi, and L.~Szpruch}, {\em Sig-sdes model for quantitative finance}, in ACM International Conference on AI in Finance, 2020.

\bibitem{barancikova2024sigdiffusions}
{\sc B.~Barancikova, Z.~Huang, and C.~Salvi}, {\em Sigdiffusions: Score-based diffusion models for long time series via log-signature embeddings}, arXiv preprint arXiv:2406.10354,  (2024).

\bibitem{batlle2024kernel}
{\sc P.~Batlle, M.~Darcy, B.~Hosseini, and H.~Owhadi}, {\em Kernel methods are competitive for operator learning}, Journal of Computational Physics, 496 (2024), p.~112549.

\bibitem{bayer2023adaptive}
{\sc C.~Bayer, S.~Breneis, and T.~Lyons}, {\em An adaptive algorithm for rough differential equations}, Weierstra{\ss}-Institut f{\"u}r Angewandte Analysis und Stochastik Leibniz-Institut~…, 2023.

\bibitem{bellingeri2022smooth}
{\sc C.~Bellingeri, P.~K. Friz, S.~Paycha, and R.~Prei{\ss}}, {\em Smooth rough paths, their geometry and algebraic renormalization}, Vietnam Journal of Mathematics, 50 (2022), pp.~719--761.

\bibitem{boedihardjo2020path}
{\sc H.~Boedihardjo, X.~Geng, and N.~P. Souris}, {\em Path developments and tail asymptotics of signature for pure rough paths}, Advances in Mathematics, 364 (2020), p.~107043.

\bibitem{brouard2016input}
{\sc C.~Brouard, M.~Szafranski, and F.~d'Alch{\'e} Buc}, {\em Input output kernel regression: Supervised and semi-supervised structured output prediction with operator-valued kernels}, Journal of Machine Learning Research, 17 (2016), p.~np.

\bibitem{cass2024weighted}
{\sc T.~Cass, T.~Lyons, and X.~Xu}, {\em Weighted signature kernels}, The Annals of Applied Probability, 34 (2024), pp.~585--626.

\bibitem{cass2025numerical}
{\sc T.~Cass, F.~Piatti, and J.~Pei}, {\em Numerical schemes for signature kernels}, arXiv preprint arXiv:2502.08470,  (2025).

\bibitem{cass_salvi_notes}
{\sc T.~Cass and C.~Salvi}, {\em Lecture notes on rough paths and applications to machine learning}, 2024, \url{https://arxiv.org/abs/2404.06583}.

\bibitem{cass2024topologies}
{\sc T.~Cass and W.~F. Turner}, {\em Topologies on unparameterised path space}, Journal of Functional Analysis, 286 (2024), p.~110261.

\bibitem{chen1957integration}
{\sc K.-T. Chen}, {\em Integration of paths, geometric invariants and a generalized baker-hausdorff formula}, Annals of Mathematics, 65 (1957), pp.~163--178.

\bibitem{chen2021solving}
{\sc Y.~Chen, B.~Hosseini, H.~Owhadi, and A.~M. Stuart}, {\em Solving and learning nonlinear pdes with gaussian processes}, Journal of Computational Physics, 447 (2021), p.~110668.

\bibitem{cirone2024theoretical}
{\sc N.~M. Cirone, A.~Orvieto, B.~Walker, C.~Salvi, and T.~Lyons}, {\em Theoretical foundations of deep selective state-space models}, arXiv preprint arXiv:2402.19047,  (2024).

\bibitem{cirone2025parallelflow}
{\sc N.~M. Cirone and C.~Salvi}, {\em Parallelflow: Parallelizing linear transformers via flow discretization}, arXiv preprint arXiv:2504.00492,  (2025).

\bibitem{cirone2025rough}
{\sc N.~M. Cirone and C.~Salvi}, {\em Rough kernel hedging}, arXiv preprint arXiv:2501.09683,  (2025).

\bibitem{cochrane2021sk}
{\sc T.~Cochrane, P.~Foster, V.~Chhabra, M.~Lemercier, T.~Lyons, and C.~Salvi}, {\em Sk-tree: a systematic malware detection algorithm on streaming trees via the signature kernel}, in 2021 IEEE international conference on cyber security and resilience (CSR), IEEE, 2021, pp.~35--40.

\bibitem{drucker1996support}
{\sc H.~Drucker, C.~J. Burges, L.~Kaufman, A.~Smola, and V.~Vapnik}, {\em Support vector regression machines}, Advances in neural information processing systems, 9 (1996).

\bibitem{dyer2022approximate}
{\sc J.~Dyer, J.~Fitzgerald, B.~Rieck, and S.~M. Schmon}, {\em Approximate bayesian computation for panel data with signature maximum mean discrepancies}, in NeurIPS 2022 Temporal Graph Learning Workshop, 2022.

\bibitem{signax}
{\sc A.~T. et~al}, {\em signax 0.2.1 - pypi}, 2024, \url{https://pypi.org/project/RoughPy/} (accessed 2024-3-24).

\bibitem{esig}
{\sc T.~L. et~al}, {\em Coropa computational rough paths (software library)},  (2010), \url{http:// coropa.sourceforge.net/}.

\bibitem{fermanian2023new}
{\sc A.~Fermanian, T.~Lyons, J.~Morrill, and C.~Salvi}, {\em New directions in the applications of rough path theory}, IEEE BITS the Information Theory Magazine,  (2023).

\bibitem{fermanian2021framing}
{\sc A.~Fermanian, P.~Marion, J.-P. Vert, and G.~Biau}, {\em Framing rnn as a kernel method: A neural ode approach}, Advances in Neural Information Processing Systems, 34 (2021), pp.~3121--3134.

\bibitem{flint2015pathwise}
{\sc G.~Flint and T.~Lyons}, {\em Pathwise approximation of sdes by coupling piecewise abelian rough paths}, arXiv preprint arXiv:1505.01298,  (2015).

\bibitem{friz2009rough}
{\sc P.~Friz and H.~Oberhauser}, {\em Rough path limits of the wong--zakai type with a modified drift term}, Journal of Functional Analysis, 256 (2009), pp.~3236--3256.

\bibitem{friz2010multidimensional}
{\sc P.~K. Friz and N.~B. Victoir}, {\em Multidimensional stochastic processes as rough paths: theory and applications}, vol.~120, Cambridge University Press, 2010.

\bibitem{gretton2012kernel}
{\sc A.~Gretton, K.~M. Borgwardt, M.~J. Rasch, B.~Sch{\"o}lkopf, and A.~Smola}, {\em A kernel two-sample test}, The Journal of Machine Learning Research, 13 (2012), pp.~723--773.

\bibitem{hoglund2023neural}
{\sc M.~Hoglund, E.~Ferrucci, C.~Hern{\'a}ndez, A.~M. Gonzalez, C.~Salvi, L.~Sanchez-Betancourt, and Y.~Zhang}, {\em A neural rde approach for continuous-time non-markovian stochastic control problems}, arXiv preprint arXiv:2306.14258,  (2023).

\bibitem{holberg2024exact}
{\sc C.~Holberg and C.~Salvi}, {\em Exact gradients for stochastic spiking neural networks driven by rough signals}, arXiv preprint arXiv:2405.13587,  (2024).

\bibitem{horvath2023optimal}
{\sc B.~Horvath, M.~Lemercier, C.~Liu, T.~Lyons, and C.~Salvi}, {\em Optimal stopping via distribution regression: a higher rank signature approach}, arXiv preprint arXiv:2304.01479,  (2023).

\bibitem{issa2024non}
{\sc Z.~Issa, B.~Horvath, M.~Lemercier, and C.~Salvi}, {\em Non-adversarial training of neural sdes with signature kernel scores}, Advances in Neural Information Processing Systems, 36 (2024).

\bibitem{jacot2018neural}
{\sc A.~Jacot, F.~Gabriel, and C.~Hongler}, {\em Neural tangent kernel: Convergence and generalization in neural networks}, Advances in neural information processing systems, 31 (2018).

\bibitem{kidger2019deep}
{\sc P.~Kidger, P.~Bonnier, I.~Perez~Arribas, C.~Salvi, and T.~Lyons}, {\em Deep signature transforms}, Advances in Neural Information Processing Systems, 32 (2019).

\bibitem{kidger2020signatory}
{\sc P.~Kidger and T.~Lyons}, {\em Signatory: differentiable computations of the signature and logsignature transforms, on both cpu and gpu}, in International Conference on Learning Representations, 2020.

\bibitem{kiraly2019kernels}
{\sc F.~J. Kir{\'a}ly and H.~Oberhauser}, {\em Kernels for sequentially ordered data}, Journal of Machine Learning Research, 20 (2019).

\bibitem{lee2017deep}
{\sc J.~Lee, Y.~Bahri, R.~Novak, S.~S. Schoenholz, J.~Pennington, and J.~Sohl-Dickstein}, {\em Deep neural networks as gaussian processes}, arXiv preprint arXiv:1711.00165,  (2017).

\bibitem{lemercier2021siggpde}
{\sc M.~Lemercier, C.~Salvi, T.~Cass, E.~V. Bonilla, T.~Damoulas, and T.~J. Lyons}, {\em Siggpde: Scaling sparse gaussian processes on sequential data}, in International Conference on Machine Learning, PMLR, 2021, pp.~6233--6242.

\bibitem{lemercier2021distribution}
{\sc M.~Lemercier, C.~Salvi, T.~Damoulas, E.~Bonilla, and T.~Lyons}, {\em Distribution regression for sequential data}, in International Conference on Artificial Intelligence and Statistics, PMLR, 2021, pp.~3754--3762.

\bibitem{li2017mmd}
{\sc C.-L. Li, W.-C. Chang, Y.~Cheng, Y.~Yang, and B.~P{\'o}czos}, {\em Mmd gan: Towards deeper understanding of moment matching network}, arXiv preprint arXiv:1705.08584,  (2017).

\bibitem{lyons2004differential}
{\sc T.~Lyons, M.~Caruana, and T.~L{\'e}vy}, {\em Differential equations driven by rough paths}, Ecole d’{\'e}t{\'e} de Probabilit{\'e}s de Saint-Flour XXXIV,  (2004), pp.~1--93.

\bibitem{lyons1998differential}
{\sc T.~J. Lyons}, {\em Differential equations driven by rough signals}, Revista Matem{\'a}tica Iberoamericana, 14 (1998), pp.~215--310.

\bibitem{manten2024signature}
{\sc G.~Manten, C.~Casolo, E.~Ferrucci, S.~W. Mogensen, C.~Salvi, and N.~Kilbertus}, {\em Signature kernel conditional independence tests in causal discovery for stochastic processes}, arXiv preprint arXiv:2402.18477,  (2024).

\bibitem{matsubara2022robust}
{\sc T.~Matsubara, J.~Knoblauch, F.-X. Briol, C.~J. Oates, et~al.}, {\em Robust generalised bayesian inference for intractable likelihoods}, Journal of the Royal Statistical Society Series B, 84 (2022), pp.~997--1022.

\bibitem{mckee2000euler}
{\sc S.~McKee, T.~Tang, and T.~Diogo}, {\em An euler-type method for two-dimensional volterra integral equations of the first kind}, IMA Journal of Numerical Analysis, 20 (2000), pp.~423--440.

\bibitem{morrill2021neural}
{\sc J.~Morrill, C.~Salvi, P.~Kidger, and J.~Foster}, {\em Neural rough differential equations for long time series}, in International Conference on Machine Learning, PMLR, 2021, pp.~7829--7838.

\bibitem{muca2023neural}
{\sc N.~Muca~Cirone, M.~Lemercier, and C.~Salvi}, {\em Neural signature kernels as infinite-width-depth-limits of controlled resnets}, arXiv e-prints,  (2023), pp.~arXiv--2303.

\bibitem{lorenzothesis}
{\sc L.~Pacchiardi}, {\em Statistical inference in generative models using scoring rules}, PhD thesis, University of Oxford, 2022.

\bibitem{pannier2024path}
{\sc A.~Pannier and C.~Salvi}, {\em A path-dependent pde solver based on signature kernels}, arXiv preprint arXiv:2403.11738,  (2024).

\bibitem{reizenstein2018iisignature}
{\sc J.~Reizenstein and B.~Graham}, {\em The iisignature library: efficient calculation of iterated-integral signatures and log signatures}, arXiv preprint arXiv:1802.08252,  (2018).

\bibitem{salvi2021rough}
{\sc C.~Salvi}, {\em Rough paths, kernels, differential equations and an algebra of functions on streams}, PhD thesis, University of Oxford, 2021.

\bibitem{salvi2021signature}
{\sc C.~Salvi, T.~Cass, J.~Foster, T.~Lyons, and W.~Yang}, {\em The signature kernel is the solution of a goursat pde}, SIAM Journal on Mathematics of Data Science, 3 (2021), pp.~873--899.

\bibitem{salvi2023structure}
{\sc C.~Salvi, J.~Diehl, T.~Lyons, R.~Preiss, and J.~Reizenstein}, {\em A structure theorem for streamed information}, Journal of Algebra, 634 (2023), pp.~911--938.

\bibitem{salvi2022neural}
{\sc C.~Salvi, M.~Lemercier, and A.~Gerasimovics}, {\em Neural stochastic pdes: Resolution-invariant learning of continuous spatiotemporal dynamics}, Advances in Neural Information Processing Systems, 35 (2022), pp.~1333--1344.

\bibitem{salvi2021higher}
{\sc C.~Salvi, M.~Lemercier, C.~Liu, B.~Horvath, T.~Damoulas, and T.~Lyons}, {\em Higher order kernel mean embeddings to capture filtrations of stochastic processes}, Advances in Neural Information Processing Systems, 34 (2021), pp.~16635--16647.

\bibitem{roughpy}
{\sc P.~R. Sam~Morley and T.~Lyons}, {\em Roughpy 0.1.0 - pypi}, 2023, \url{https://pypi.org/project/RoughPy/} (accessed 2023-12-8).

\bibitem{scholkopf2001estimating}
{\sc B.~Sch{\"o}lkopf, J.~C. Platt, J.~Shawe-Taylor, A.~J. Smola, and R.~C. Williamson}, {\em Estimating the support of a high-dimensional distribution}, Neural computation, 13 (2001), pp.~1443--1471.

\bibitem{shmelev2024sparse}
{\sc D.~Shmelev and C.~Salvi}, {\em Sparse signature coefficient recovery via kernels}, arXiv preprint arXiv:2412.08579,  (2024).

\bibitem{steinwart2008support}
{\sc I.~Steinwart and A.~Christmann}, {\em Support vector machines}, Springer Science \& Business Media, 2008.

\bibitem{toth2020bayesian}
{\sc C.~Toth and H.~Oberhauser}, {\em Bayesian learning from sequential data using gaussian processes with signature covariances}, in International Conference on Machine Learning, PMLR, 2020, pp.~9548--9560.

\bibitem{williams2006gaussian}
{\sc C.~K. Williams and C.~E. Rasmussen}, {\em Gaussian processes for machine learning}, vol.~2, MIT press Cambridge, MA, 2006.

\bibitem{wynne2022kernel}
{\sc G.~Wynne and A.~B. Duncan}, {\em A kernel two-sample test for functional data}, The Journal of Machine Learning Research, 23 (2022), pp.~3159--3209.

\end{thebibliography}

\appendix

\section{Error estimates proofs}\label{sec:proof}

\subsection{Local error estimates}\label{ssec:local}

\begin{proof}[Proof of Lemma~\ref{lemma:local_error_smooth}]
To ease notation we remove indexing by $[s,t]$. By submultiplicativity of the $\ell_1$-norm
    \begin{align*}
    \|\tensora\tensorprod S(x)_{s,t} - \mathcal{L}^{n}_{s,t}(\tensora, x)\|_1 &= \left\| \tensora \tensorprod S(x) - \tensora \tensorprod\exp(\log_nS(x)) \right\|_1 \\
    &\leq  \| \tensora\|_1\left\| \exp(\log S(x)) - \exp(\log_n S(x)) \right\|_1
    \end{align*}
Then, we have to bound the error 
\begin{align*}
    E&:=\left\| \exp(\log S(x)) - \exp(\log_n S(x)) \right\|_1.
\end{align*}
Using the definition of $\exp$ and $\log$, we obtain
\begin{align*}
    E&= \left\| \sum_{k=0}^\infty \frac{1}{k!} \log S(x)^{\cdot k} - (\log_n S(x))^{\cdot k}\right\|_1 \\
    &=  \left\| \sum_{k=0}^\infty \frac{1}{k!} \left(\sum_{i=1}^\infty \pi_i \log S(x)\right)^{\cdot k} - \left(\sum_{i=1}^n \pi_i \log S(x)\right)^{\cdot k} \right\|_1 \\
    &=  \left\| \sum_{k=0}^\infty \frac{1}{k!} \sum_{j=0}^{k-1} C^k_j \left(\sum_{i=1}^n \pi_i \log S(x) \right)^{\cdot j} \left(\sum_{i=n+1}^\infty \pi_i \log S(x)\right)^{\cdot (k-j)} \right\|_1 \\
    &\leq \sum_{k=0}^\infty \frac{1}{k!} \sum_{j=0}^{k-1} C^k_j \left(\sum_{i=1}^n \|\pi_i\log S(x)\|_{\statespace^{\otimes i}}\right)^{j} \left(\sum_{i=n+1}^\infty \|\pi_i\log S(x)\|_{\statespace^{\otimes i}}\right)^{k-j} 
    \\ 
    &\leq \sum_{k=0}^\infty \frac{1}{k!} \sum_{j=0}^{k-1} C^k_j \left(\sum_{i=1}^n (e\|x\|_{1\text{-var}})^i\right)^{j} \left(\sum_{i=n+1}^\infty (e\|x\|_{1\text{-var}})^i\right)^{k-j} \\
    &=  \sum_{k=0}^\infty \frac{1}{k!} \left(\left(\sum_{i=1}^\infty (e\|x\|_{1\text{-var}})^i\right)^k - \left(\sum_{i=1}^n (e\|x\|_{1\text{-var}})^i\right)^k\right) \\
    &= \left(\exp\left(\sum_{i=1}^\infty (e\|x\|_{1\text{-var}})^i\right) -\exp\left(\sum_{i=1}^n (e\|x\|_{1\text{-var}})^i\right) \right) \\ 
    &=  \left(\exp\left(\frac{e\|x\|_{1\text{-var}}}{1 - e\|x\|_{1\text{-var}}}\right) -\exp\left(\frac{e\|x\|_{1\text{-var}}(1 - (e\|x\|_{1\text{-var}})^n)}{1 - e\|x\|_{1\text{-var}}}\right) \right) \\
    &= \exp \left(\frac{e\|x\|_{1\text{-var}}(1 - (e\|x\|_{1\text{-var}})^n)}{1 - e\|x\|_{1\text{-var}}}\right)\left(\exp\left(\frac{(e\|x\|_{1\text{-var}})^{n+1}}{1 - e\|x\|}\right) - 1\right) \\
    &\leq \frac{1}{1 - e\|x\|_{1\text{-var}}}\exp\left(\frac{e\|x\|_{1\text{-var}} + (e\|x\|_{1\text{-var}})^{n+1}}{1 - e\|x\|_{1\text{-var}}}\right) (e\|x\|_{1\text{-var}})^{n+1}  \\
    &\leq \frac{1}{1 - e\|x\|_{1\text{-var}}}\exp\left(\frac{2e\|x\|_{1\text{-var}}}{1 - e\|x\|_{1\text{-var}}}\right) (e\|x\|_{1\text{-var}})^{n+1}  \\
    &:= C (e\|x\|_{1\text{-var}})^{n+1}
\end{align*}
where in the penultimate inequality we used the simple bounds
\begin{equation*}
    \exp \left(\frac{e\|x\|_{1\text{-var}}(1 - (e\|x\|_{1\text{-var}})^n)}{1 - e\|x\|_{1\text{-var}}}\right) \leq \exp \left(\frac{e\|x\|_{1\text{-var}}}{1 - e\|x\|_{1\text{-var}}}\right), \quad \text{and}
\end{equation*}
\begin{equation*}
    \exp\left(\frac{(e\|x\|_{1\text{-var}})^{n+1}}{1 - e\|x\|_{1\text{-var}}}\right) - 1 \leq \frac{(e\|x\|_{1\text{-var}})^{n+1}}{1 - e\|x\|_{1\text{-var}}} \exp\left(\frac{(e\|x\|_{1\text{-var}})^{n+1}}{1 - e\|x\|_{1\text{-var}}}\right)
\end{equation*}
and earlier we used the fact that for any $i \geq 1$
\begin{equation}
    \pi_i \log S(x) = \sum_{j=1}^i \frac{(-1)^j}{j}\sum_{i_1 + ... + i_j = i} \pi_{i_1} S(x)\otimes ... \otimes \pi_{i_j} S(x),
\end{equation}
therefore, by the signature factorial decay and then using the multinomial theorem we get
\begin{align}
\| \pi_i \log S(x) \|_{\statespace^{\otimes i}} &\leq  \sum_{j=1}^i \frac{1}{j} \sum_{i_1 + ... + i_j = i} \frac{\|x\|^{i_1}_{1\text{-var}}...\|x\|_{1\text{-var}}^{i_j}}{i_1!...i_j!} \\
&= \sum_{j=1}^i \frac{1}{ji!}(j\|x\|_{1\text{-var}})^i \leq \|x\|_{1\text{-var}}^i\frac{i^i}{i!} \leq (e\|x\|_{1\text{-var}})^i.
\end{align}
where we used the classical inequality $i! \geq (i/e)^i$.
\end{proof}

\begin{lemma}\cite[Exercise 3.9]{lyons2004differential}\label{neoclassical_inequality}
For any $p \in [1,\infty)$, $n \in \mathbb{N}$, and $x_1,...,x_j \ge 0,$ 
\begin{equation}
    \frac{1}{p^{2j-2}} \sum_{i_1 + ... + i_j = i} \frac{x_1^{i_1/p}... x_j^{i_j/p}}{\left(i_1/p\right)! ... \left(i_j/p\right)!} \leq \frac{(x_1+...+x_j)^{\frac{i}{p}}}{\left(i/p\right)!},
\end{equation}
\end{lemma}

\begin{proof}[Proof of Lemma~\ref{lemma:local_error}]
Using the inequality (\ref{p-var_bound}) as well as \Cref{neoclassical_inequality}, for any $i \in \mathbb N$ we have
\begin{align*}
    \| \pi_i \log S(\tensor{X}) \|_{\statespace^{\otimes i}} &= \left\|\sum_{j=1}^i \frac{(-1)^j}{j}\sum_{i_1 + ... + i_j = i} \pi_{i_1} S(\tensor{X})\otimes ... \otimes \pi_{i_j} S(\tensor{X})\right\|_{\statespace^{\otimes i}}\\
    &\leq \sum_{j=1}^i \frac{(-1)^j}{j}\sum_{i_1 + ... + i_j = i}  \|\pi_{i_1} S(\tensor{X})\|_{\statespace^{\otimes i_1}} ...  \|\pi_{i_j}S(\tensor{X})\|_{\statespace^{\otimes i_j}}\\
    &\leq  \sum_{j=1}^i \frac{1}{j} \sum_{i_1 + ... + i_j = i} \frac{1}{i_1!...i_j!}\frac{\omega(s,t)^{i_1/p}}{\beta(i_1/p)!}...\frac{\omega(s,t)^{i_j/p}}{\beta(i_j/p)!} \\
    &\leq \sum_{j=1}^i\frac{1}{j}\frac{(j\omega(s,t))^{i/p}}{(i/p)!} \\
    &\leq \omega(s,t)^{i/p} \frac{i^{i/p}}{(i/p)!} \\
    &\leq (e\omega(s,t))^{i/p}
\end{align*}
where we used the classical inequality $x! \geq (x/e)^x$ for any $x>0$. Repeating the same computations as in the proof of \Cref{lemma:local_error_smooth} yields the result.
\end{proof}

\begin{lemma}\label{lemma:bound_local_logODE_sol} Let $\mathbb X\in G\Omega_p(\statespace)$ be a geometric $p$-rough path, and $\tensora$ an element of $T((\statespace))$. Then, 
   \begin{equation}
   \|\mathcal{L}^{n}_{s,t}(\tensora;\tensor{X})\|_1 \leq \|\tensora\|_1\prod_{\ell=1}^{n} \exp((ew(s,t))^{1/p})
   \end{equation}
\end{lemma}
\begin{proof}[Proof of Lemma~\ref{lemma:bound_local_logODE_sol}]
Using similar calculations as in the proof of \Cref{lemma:local_error}
\begin{align*}
    \|\mathcal{L}^{n}_{s,t}(\tensora;\tensor{X})\|_1&=\|\tensora\tensorprod\exp(\log_n S(x)_{s,t})\|_1\\
    &\leq \|\tensora\|_1\left\| \sum_{k=0}^{\infty}\frac{1}{k!}(\log_n S(x)_{s,t})^k\right\|_1\\
    &\leq \|\tensora\|_1\sum_{k=0}^{\infty}\frac{1}{k!}\left\|\left(\sum_{\ell=1}^{n}\pi_\ell \log S(x)_{s,t}\right)^k\right\| \\ 
    &\leq \|\tensora\|_1\sum_{k=0}^{\infty}\frac{1}{k!}\left(\sum_{\ell=1}^{n}\|\pi_\ell \log S(x)_{s,t}\|\right)^k\\ 
    &\leq \|\tensora\|_1\sum_{k=0}^{\infty}\frac{1}{k!}\left(\sum_{\ell=1}^{n}(ew(s,t))^{\ell/p}\right)^k\\ 
    &= \|\tensora\|_1\exp\left(\sum_{\ell=1}^{n}(ew(s,t))^{\ell/p}\right) \\
    &= \|\tensora\|_1\prod_{\ell=1}^{n}\exp(ew(s,t))^{\ell/p})
\end{align*}
\end{proof}

\subsection{Global error estimate}\label{ssec:global}

\begin{proof}[Proof of Lemma~\ref{lemma:global_error}]
Consider the sequence $(a_i)_{i=1}^{N}$ with main term defined by
\begin{align*}
    a_i = \prod_{k=1}^{i-1}\exp(\log_n S(x)_{t_k,t_{k+1}})\tensorprod\prod_{k=i}^{N-1}
 S(x)_{t_{k},t_{k+1}}.
\end{align*}
This is the solution of the signature RDE on $[t_i, t_N]$ started at the solution of the step-$n$ log-ODE at time $t_i$. The first term is given by 
\begin{align*}
    a_1&=\prod_{k=1}^{N-1}S(x)_{t_{k},t_{k+1}}= S(x)_{t_{1}, t_N}
\end{align*}
the signature of $x$ on the interval $[t_1, t_N]$. The last term is given by
\begin{align*}
    a_{N}&=\prod_{k=1}^{N-1}\exp(\log_n S(x)_{t_k,t_{k+1}}) = \widetilde{Z}_{N}
\end{align*}
the solution of the step-$n$ log-ODE on $[t_1, t_N]$.
For any $i=1, \ldots, N-1$, the main term of the sequence can be rewritten as
\begin{align*}
a_{i} &= \prod_{k=1}^{i-1}\exp(\log_n S(x)_{t_k,t_{k+1}})\tensorprod S(x)_{t_{i}, t_{i+1}}
 \tensorprod S(x)_{t_{i+1},t_{N}}
 \end{align*}
 where we have used Chen's identity on the signature of $x$ between $[t_i, t_{i+1}]$ and $[t_{i+1}, t_N]$, and the fact that $S_{t_N, t_N}=1$. 
 Similarly, for any $i=1, \ldots, N-1$, the next term of the sequence is the solution of the signature RDE on $[t_{i+1}, t_N]$ started at $t_{i+1}$ at the solution of the step-$n$ log-ODE at time $t_{i+1}$. This time, we rewrite the solution $\widetilde{Z}_{i+1}$ of the step-$n$ log-ODE on $[t_1, t_{i+1}]$ as the solution of the step-$n$ log-ODE on $[t_{i}, t_{i+1}]$, started at $\widetilde{Z}_{i}$ at $t_i$, so that the term $a_{i+1}$ reads as:
 \begin{align*}
    a_{i+1} &= \prod_{k=1}^{i}\exp(\log_n S(x)_{t_k,t_{k+1}})
 \prod_{k=i+1}^{N-1}
 S(x)_{t_{k},t_{k+1}}\\
 &=\prod_{k=1}^{i-1}\exp(\log_n S(x)_{t_k,t_{k+1}})
 \tensorprod \exp(\log_n S(x)_{t_i,t_{i+1}})\tensorprod S(x)_{t_{i+1},t_N}  
\end{align*}
Therefore, the difference between the solution $S(x)_{0,T}$ of the signature RDE and the step-$n$ log-ODE solution can be rewritten as a telescopic sum which reads as: 
\begin{align*}
   E&:=S(x)_{0,T}-\widetilde{Z}_{N} \\
   &= \sum_{i=1}^{N-1}(a_i - a_{i+1})\\
 &=\sum_{i=1}^{N-1}\prod_{k=1}^{i-1}\exp(\log_n S(x)_{t_i,t_{i+1}})
 \tensorprod\left(S(x)_{t_i,t_{i+1}}-\exp(\log_n S(x)_{t_i,t_{i+1}})\right)\tensorprod S(x)_{t_{i+1},t_N} 
\end{align*}
Taking the norm we get
\begin{align*}
   \|S(x)_{0,T}-\widetilde{Z}_{N} \|&\leq \sum_{i=0}^{N-1} \| \widetilde{Z}_{i}\| \|S(x)_{t_i, t_{i+1}}-\exp(\log_n S(x))_{t_i,t_{i+1}}\|\|S(x)_{t_{i+1},t_N} \|
\end{align*}
By \Cref{lemma:bound_local_logODE_sol}, we obtain the following bound on the solution of the step-$n$ log-ODE on the time interval $[0,t_i]$
\begin{align}
    \|\widetilde{Z}_{i}\|&=\|\prod_{k=1}^{i-1} \exp(\log_nS(x)_{t_k,t_{k+1}})\|\\
    &\leq \prod_{k=1}^{i-1}\| \log_n S(x)_{t_k,t_{k+1}})\| \\
    &\leq \prod_{k=1}^{i-1}\prod_{\ell=1}^{n} \exp((ew(t_k,t_{k+1}))^{1/p})
\end{align}
By the extension theorem, the norm of the signature of $x$ on $[t_{i+1}, t_N]$ is bounded by
\begin{align*}
    \|S(x)_{t_{i+1},t_N}\|\leq \sum_{k=0}^{\infty}\frac{(w(t_{i+1},t_N))^{k/p}}{\beta_p(k/p)!} \leq \exp(w(0,T);p)
\end{align*}
Applying \Cref{lemma:local_error}, we have  
\begin{align*}
   \|S(x)_{s,t}-\exp(\log_n S(x)_{s,t)}\|&\leq \frac{1}{1 - (ew(s,t))^{\frac{1}{p}}}\exp\left(\frac{2(ew(s,t))^{\frac{1}{p}}}{1 - (ew(s,t))^{\frac{1}{p}}}\right) (ew(s,t))^{\frac{n+1}{p}} \\ 
   & \leq \max_{i=1, \ldots, N} c_i(ew(s,t))^{\frac{n+1}{p}-1} w(s, t)
\end{align*}
Therefore, 
\begin{align*}
   \|S(x)_{0,T}-\widetilde{Z}_N \|&\leq \exp(w(0,T);p)^2\max_{i=1, \ldots, N} c_i(ew(t_i,t_{i+1}))^{\frac{n+1}{p}-1} \sum_{i=0}^{N-1} w(t_i, t_{i+1}) \\ 
   & \leq w(0,T)\exp(w(0,T);p)^2\max_{i=1, \ldots, N}c_i\max_{i=1, \ldots, N} (ew(t_i,t_{i+1}))^{\frac{n+1}{p}-1} \\
   &\leq C\max_{i=1, \ldots, N} (ew(t_i,t_{i+1}))^{\frac{n+1}{p}-1}
\end{align*}
We note that, using \Cref{eq:partition_control} below, we can choose at tbe beginning a partition so that $w(t_i,t_{i+1})\leq \frac{w(0,T)}{N}$ for all $t_i$.
\end{proof}

\begin{lemma}\label{eq:partition_control}
Let $w$ be a control on $[s,t]$. For all $N\geq 1$, there exists a partition $D=\{s=t_0< t_1< t_2< \ldots< t_N=t\}$ so that $w(t_i,t_{i+1})\leq \frac{w(s,t)}{N}$ for all $i=0,\ldots, N-1$.
\end{lemma}
\begin{proof}
We proceed by induction. It is clearly true when $N=1$ as we can take $D=\{s,t\}$. Consider the case where $N>1$. Consider the function $\phi(u)=w(s,u)$. The function $\phi$ is continuous, $\phi(s)=0$ and $\phi(t)=w(s,t)$. Hence, by the intermediate value theorem, there exists a $u$ such that $\phi(u)=w(s,t)/N$. By the superadditivity of the control $w$, we have $w(s,u)+w(u,t)\leq w(s,t)$. Therefore, 
$$w(u,t)\leq w(s,t)-\frac{w(s,t)}{N}=\frac{N-1}{N} w(s,t).$$ By the induction hypothesis, $[u,t]$ can be broken into $N-1$ intervals $\{u=t_1< \ldots < t_{N}=t\}$ such that $$w(t_i,t_{i+1})\leq \frac{w(u,t)}{N-1}\leq \frac{w(s,t)}{N}.$$ Setting $t_0=s$ we have proved the theorem for this $N$ and hence for all $N$. 
\end{proof}

\section{Main results}\label{ssec:main_results_proofs}

\begin{proof}[Proof of Lemma~\ref{thm:pde}]
Omitting all indexing intervals $[s_1,s_2]$ and $[t_1,t_2]$ to ease notation we have
\begin{align*}
    &\left\langle \tensor{Z}^x_u - \tensor{Z}_0^x ,\tensor{Z}^y_v - \tensor{Z}_0^y \right\rangle \\
    &= f_{u,v} - f_{u,0} - f_{0,v} + f_{0,0}\\
    &= 
    \Big\langle \!\int_0^u\!\! \tensor{Z}_p^{x} \cdot \Big[\sum_{i=1}^d x^{(i)}e_i + \!\!\sum_{i,j=1}^d \!x^{(i,j)} e_i\otimes e_j  \Big]dp, \!\int_0^v \!\!\tensor{Z}_q^{y} \cdot \Big[\sum_{i=1}^d\!y^{(i)}e_i + \!\!\sum_{i,j=1}^dy^{(i,j)} e_i\otimes e_j \Big]dq \!\Big\rangle \\  
    &= \int_0^u \int_0^v \Big[ \gamma f_{p,q}   + \sum_{i,j,k=1}^d x^{(i,j)} y^{(k)}\bigl\langle \tensor{Z}_p^{x} \cdot e_i \otimes e_j, \tensor{Z}_q^{y} \cdot e_k \bigr\rangle   \Big] dpdq \\
    &\quad\quad  + \int_0^u \int_0^v \Big[\sum_{i,j,k=1}^d x^{(i)}y^{(j,k)} \bigl\langle \tensor{Z}_p^{x} \cdot e_i,  \tensor{Z}_q^{y} \cdot e_j \otimes e_k \bigr\rangle \Big] dpdq\\
    &= \int_0^u \int_0^v \Big[ \gamma f_{p,q}   + \sum_{i,j=1}^dx^{(i,j)} y^{(j)} \varphi^{(i)}_{p,q} + \sum_{i,j=1}^d x^{(i)} y^{(i,j)} \psi^{(j)}_{p,q} \Big]dpdq \\
    &= \int_0^u \int_0^v \Big[\gamma f_{p,q} + \sum_{j=1}^d \alpha_j \varphi^{(j)}_{p,q} +  \sum_{j=1}^d \beta_j \psi^{(j)}_{p,q} \Big]dpdq 
\end{align*}
where 
\begin{align}
    f_{p,q}:= \left\langle \tensor{Z}_p^{x}, \tensor{Z}^y_q\right\rangle, \quad \varphi^{(i)}_{p,q}:= \left\langle \tensor{Z}_p^{x} \tensorprod e_i, \tensor{Z}_q^{y} \right\rangle, \quad \psi^{(i)}_{p,q}:= \left\langle \tensor{Z}_p^{x}, \tensor{Z}_q^{y} \tensorprod e_i \right\rangle
\end{align}
and 
\begin{equation*}
    \alpha_j := \sum_{i=1}^d x^{(i,j)}y^{(i)}, \quad \beta_j := \sum_{i=1}^d x^{(i)}y^{(i,j)}, \quad \gamma:=\sum_{i=1}^{d}x^{(i)}y^{(i)}+\sum_{i=1}^{d}\sum_{j=1}^{d}x^{(i,j)}y^{(i,j)}.
\end{equation*}

\medskip 
\noindent Now, for any $r \in \{1, ..., d\}$ we have
\begin{align*}
    \left\langle \tensor{Z}_u^{x} \tensorprod e_r, \tensor{Z}_v^{y} - \tensor{Z}_0^{y} \right\rangle &= \varphi^{(r)}_{u,v} - \varphi^{(r)}_{u,0} \\
    &= \Big\langle \tensor{Z}_u^{x}  \tensorprod e_r, \int_0^v \tensor{Z}_q^{y} \cdot \sum_{i=1}^d\sum_{j=1}^d\left( y^{(i)}e_i + y^{(i,j)} e_i\otimes e_j\right) dq\Big\rangle \\
    &= \int_0^v \Big(y^{(r)} f_{u,q} + \sum_{i=1}^d y^{(i,r)} \psi^{(i)}_{u,q} \Big) dq 
\end{align*}
Similarly
\begin{align*}
    \left\langle \tensor{Z}_u^{x} - \tensor{Z}_0^{x}, \tensor{Z}_v^{y}\cdot e_r \right\rangle &= \psi^{(r)}_{u,v} - \psi^{(r)}_{0,v} = \int_0^u \Big( x^{(r)}f_{p,v} + \sum_{i=1}^d x^{(i,r)} \varphi^{(i)}_{p,v} \Big) dp.
\end{align*}
Differentiating with respect to $u$ and $v$ we obtain the aforementioned system of PDEs.
\end{proof}

\begin{proof}[Proof of \Cref{thm:global_error_second_order_pde}]
Denote by $\widetilde{Z}_{N}, \widetilde{Z}_{N'}$ the solutions of the step-$n$ log-ODEs over $[0,T]$. By Cauchy-Schwarz, \Cref{lemma:global_error} and \Cref{lemma:bound_local_logODE_sol}
\begin{align*}
    |\langle S(\tensor{X}), S(\tensor{Y})\rangle - f(T, T)| &= |\langle S(\tensor{X}) - \widetilde{Z}_{N}, \widetilde{Z}_{N'} \rangle + \langle S(\tensor{Y}) - \widetilde{Z}_{N'}, \widetilde{Z}_{N} \rangle| \\
    &\leq |\langle S(\tensor{X})_{0,T} - \widetilde{Z}_{N}, \widetilde{Z}_{N'} \rangle | + |\langle S(\tensor{Y}) - \widetilde{Z}_{N'} , \widetilde{Z}_{N} \rangle| \\
    &\leq \|S(\tensor{X})- \widetilde{Z}_{N}\|_2 \|\widetilde{Z}_{N'}\|_2 + \|S(\tensor{Y}) - \widetilde{Z}_{N'}\|_2 \|\widetilde{Z}_{N}\|_2 \\
    &\leq \|S(\tensor{X}) - \widetilde{Z}_{N}\|_1 \|\widetilde{Z}_{N'}\|_1 + \|S(\tensor{Y}) - \widetilde{Z}_{N'}\|_1 \|\widetilde{Z}_{N}\|_1 \\
    &\leq C_1 \max_{i=0, \ldots, N-1} (e\omega_x(s_i,s_{i+1}))^{(n+1)/p-1} \\
    &\quad + C_2 \max_{i=0, \ldots, N'-1} (e\omega_y(t_i,t_{i+1}))^{(n+1)/p-1} \\
    &\leq 2\max\big\{C_1\max_{i=0, \ldots, N-1} (e\omega_x(s_i,s_{i+1}))^{(n+1)/p-1}, \\&\quad C_2\max_{i=0, \ldots, N'-1} (e\omega_y(t_i,t_{i+1}))^{(n+1)/p-1}\big\},
\end{align*}
where we used the fact that $\|\cdot\|_2 \leq \| \cdot \|_1$.
\end{proof}

\begin{proof}[Proof of \Cref{thm:uniqueness}]
    Existence follows by \Cref{thm:pde}. To prove uniqueness, it suffices to prove that (\ref{eqn_phi}) has a unique solution in $C([0,1]^2,\mathbb R^{2d + 1})$. We follow a similar strategy as in \cite{mckee2000euler}. By Picard iteration, define the sequence $(\Phi_n)$ in $C([0,1]^2,\mathbb R^{2d + 1})$ recursively as follows
    \begin{equation}
        \Phi_{n+1}(s,t) = h_0(s,t) + \mathcal{T}\Phi_{n}(s,t), \quad \Phi_0 \equiv 0.
    \end{equation}
    Let 
    \[
    \delta := K_4 + K_5 + K_6 + K_7 + \sqrt{\bigl(K_4 + K_5 + K_6+K_7\bigr)^2 + 2\bigl(K_1 + K_2+K_3\bigr)}.
    \]
    Then it's easy to verify that
    \begin{equation}
        \frac{K_1 + K_2 + K_3}{\delta^2} + \frac{K_4 + K_6}{\delta} + \frac{K_5 + K_7}{\delta} = \frac{1}{2}
    \end{equation}
    We show by induction that for any $s,t$
    \begin{equation}\label{induction}
        \|\Phi_{n}(s,t) - \Phi_{n-1}(s,t) \| \leq h_0 \left(\frac{1}{2}\right)^{n-1} e^{\delta(s+t)}
    \end{equation}
    As $\Phi_1 = h$, (\ref{induction}) trivially holds for $n=1$. Assume  (\ref{induction}) holds, then
    \begin{align*}
        &\|\Phi_{n+1}(s,t) - \Phi_{n}(s,t)\| = \| \mathcal{T}\Phi_{n}(s,t) - \mathcal{T}\Phi_{n-1}(s,t)\| \\
        &\leq K_1 \int_0^s \int_0^t |[\Phi_{n}(u,v) - \Phi_{n-1}(u,v)]_1| dudv \\
        &+ K_2 \int_0^s \int_0^t \|[\Phi_{n}(u,v) - \Phi_{n-1}(u,v)]_{2:d+1}\| dudv  \\
        &+ K_3 \int_0^s \int_0^t \|[\Phi_{n}(u,v) - \Phi_{n-1}(u,v)]_{d+2:2d+1}\| dudv  \\
        &+ K_5 \int_0^t |[\Phi_{n}(s,v) - \Phi_{n-1}(s,v)]_1| dv + K_7 \int_0^t \|[\Phi_{n}(s,v) - \Phi_{n-1}(s,v)]_{d+2:2d+1}\|dv\\
        &+ K_4 \int_0^s |[\Phi_{n}(u,t) - \Phi_{n-1}(u,t)]_1| du + K_6 \int_0^s \|[\Phi_{n}(u,t) - \Phi_{n-1}(u,t)]_{2:d+1}\|du \\
        &\leq h_0(K_1 + K_2 + K_3) \int_0^s \int_0^t \left(\frac{1}{2}\right)^{n-1} e^{\delta(u + v)} dud \\
        &+ h_0(K_4 + K_6) \int_0^s \left(\frac{1}{2}\right)^{n-1} e^{\delta(u+t)}du + h_0(K_5 + K_7) \int_0^t \left(\frac{1}{2}\right)^{n-1} e^{\delta(s+v)}dv \\
        &= h_0 \left(\frac{1}{2}\right)^{n-1} \bigg(\frac{K_1 + K_2 + K_3 }{\delta^2} (e^{\delta s} - 1)(e^{\delta t} - 1) + \frac{K_4 + K_6}{\delta}e^{\delta t}(e^{\delta s} - 1) \\
        &+ \frac{K_5 + K_7}{\delta}e^{\delta s}(e^{\delta t} - 1) \bigg) \\
        &\leq h_0 \left(\frac{1}{2}\right)^{n-1} \left(\frac{K_1 + K_2 + K_3}{\delta^2} + \frac{K_4 + K_6}{\delta} + \frac{K_5 + K_7}{\delta} \right) e^{s + t} \\
        &= h_0 \left(\frac{1}{2}\right)^{n}  e^{s + t}  
    \end{align*}
    Thus $(\Phi_n)$ is a Cauchy sequence in the  Banach space $C([0,1]^2,\mathbb R^{2d + 1})$, which implies that $\Phi = \lim_{n \to \infty} \Phi_n$ is an element of $C([0,1]^2,\mathbb R^{2d + 1})$. Noting that each $\Phi_n$ is bounded by
    \begin{align*}
        \|\Phi_n(s,t)\| \leq \|\sum_{k=0}^{n-1}\Phi_k(s,t) - \Phi_{k-1}(s,t)\| \leq h_0e^{\delta(s+t)}\sum_{k=0}^n \left(\frac{1}{2}\right)^k \leq 2h_0e^{\delta(s+t)}
    \end{align*}
    which is integrable on $[0,1]^2$, it follows by dominated convergence that the limit $\Phi$ also satisfies equation (\ref{eqn_phi}). We now assume there exists another solution $\Phi^*$ and show by induction that
    \begin{equation}\label{sec:ind}
        \|\Phi_n(s,t) - \Phi^*(s,t)\| \leq \left(\frac{1}{2}\right)^n \|\Phi^*\|_\infty e^{\delta(s + t)}.
    \end{equation}
    The case $n=0$ follows immediately since $\Phi_0 \equiv 0$. Assume (\ref{sec:ind}) holds. Then
    \begin{align*}
        &\|\Phi_{n+1}(s,t) - \Phi^*(s,t)\| = \| \mathcal{T}\Phi_{n}(s,t) - \mathcal{T}\Phi^*(s,t)\| \\
        &\leq K_1 \int_0^s \int_0^t |[\Phi_{n}(u,v) - \Phi^*(u,v)]_1| dudv \\
        &+ K_2 \int_0^s \int_0^t \|[\Phi_{n}(u,v) - \Phi^*(u,v)]_{2:d+1}\| dudv  \\
        &+ K_3 \int_0^s \int_0^t \|[\Phi_{n}(u,v) - \Phi^*(u,v)]_{d+2:2d+1}\| dudv  \\
        &+ K_5 \int_0^t |[\Phi_{n}(s,v) - \Phi^*(s,v)]_1| dv + K_7 \int_0^t \|[\Phi_{n}(s,v) - \Phi^*(s,v)]_{d+2:2d+1}\|dv\\
        &+ K_4 \int_0^s |[\Phi_{n}(u,t) - \Phi^*(u,t)]_1| du + K_6 \int_0^s \|[\Phi_{n}(u,t) - \Phi^*(u,t)]_{2:d+1}\|du \\
        &\leq \|\Phi^*\|_\infty(K_1 + K_2 + K_3) \int_0^s \int_0^t \left(\frac{1}{2}\right)^{n-1} e^{\delta(u + v)} dud \\
        &+ \|\Phi^*\|_\infty(K_4 + K_6) \int_0^s \left(\frac{1}{2}\right)^{n-1} e^{\delta(u+t)}du + \|\Phi^*\|_\infty(K_5 + K_7) \int_0^t \left(\frac{1}{2}\right)^{n-1} e^{\delta(s+v)}dv \\
        &= \|\Phi^*\|_\infty \left(\frac{1}{2}\right)^{n-1} \bigg(\frac{K_1 + K_2 + K_3 }{\delta^2} (e^{\delta s} - 1)(e^{\delta t} - 1) + \frac{K_4 + K_6}{\delta}e^{\delta t}(e^{\delta s} - 1)  \\
        &+\frac{K_5 + K_7}{\delta}e^{\delta s}(e^{\delta t} - 1) \bigg) \\
        &\leq \|\Phi^*\|_\infty \left(\frac{1}{2}\right)^{n-1} \left(\frac{K_1 + K_2 + K_3}{\delta^2} + \frac{K_4 + K_6}{\delta} + \frac{K_5 + K_7}{\delta} \right) e^{s + t} \\
        &= \|\Phi^*\|_\infty \left(\frac{1}{2}\right)^{n}  e^{s + t}  
    \end{align*}
    Letting $n \to \infty$ yields $\Phi = \Phi^*$.
\end{proof}

\section{The step-$n$ log-PDE method}\label{sec:generic_n}
\subsection{Adjoint of tensor multiplication in the extended tensor algebra}\label{ssec:adjoints_proofs}
\begin{proposition}\label{prop:adjoint}Let $\tensora\in T((V))$, $u\in \statespace^{\otimes p}$ and $v\in \statespace^{\otimes q}$ with $p\leq q$. Then, 
\begin{align}
\radj{u}{\tensora\tensorprod v} = \tensora\tensorprod \radj{u}{v}
\end{align}
\end{proposition}
\begin{proof} Take $u=e_{i_1}\otimes\ldots \otimes e_{i_p} $, $v=e^*_{j_1}\otimes \ldots \otimes e^*_{j_q}$ and $a=e^*_{k_1}\otimes \ldots \otimes e^*_{k_n}$. Applying the definition of the adjoint on basis elements:
\begin{equation}
    \radjmap{u}(a\otimes e^*_{j_1}\otimes\ldots \otimes e^*_{j_q})=\left\{
  \begin{array}{ll}
    a\otimes e^*_{j_{1}}\otimes\ldots\otimes e^{*}_{j_{q-p-1}} & \text{if } n\geq p,~ j_{q-p}=i_1, \ldots, j_q=i_p\\
    0,  & \mbox{otherwise}.
  \end{array}
\right.
\end{equation}
Therefore $\radj{u}{a\otimes v} = a\otimes \radj{u}{v}$. The result follows by linearity.
\end{proof}

\begin{proposition}
    Let $u\in \statespace^{\otimes p}$ and $v\in \statespace^{\otimes q}$ with $p\leq q$. Then, $$\|\radj{u}{v}\|\leq \|u\|_{\statespace^{\otimes p}}\|v\|_{\statespace^{\otimes q}}$$
\end{proposition}
\begin{proof}
By definition, we have
$$\| \radj{u}{v} \| := \sup_{\tensora\in T^{q -p} (\statespace) : \|\tensora\|= 1} [\tensora, \radj{u}{v} ]$$
and  $$[ \tensora\tensorprod u, v] = [\tensora, \radj{u}{v}], \quad \forall \tensora \in T((\statespace)).$$ 
Therefore, 
\begin{align*}
  \sup_{\tensora \in T^{p-a}(\statespace) : \|\tensora\|= 1} [\tensora, \radj{u}{v} ] & = \sup_{\tensora \in T^{q - p}
  (\statespace) : \|\tensora\|= 1} [\tensora\tensorprod u,v]\\
  & \leq \sup_{\tensora\in T^{q - p}(\statespace) : \|\tensora\|= 1} \|\tensora\tensorprod u\|\|v\|\\
  & \leq \sup_{\tensora\in T^{q - p}(\statespace) : \|\tensora\|= 1} \|\tensora\|\|u\| \|v\|\\
  & \leq \|u\|\|v\|
\end{align*}
Therefore, $\| \radj{u}{v} \| \leq \|u\|\|v\|$.
\end{proof}

\subsection{Proof of \Cref{thm:pde_genericn}}
Denote the log signatures of $\mathbb X$ and $\mathbb Y$ over the subintervals $[s_1,s_2]$ and $[t_1,t_2]$ of $[0,T]$ by $\log_nS(\tensor{X})_{s_1,s_2}=(\pab{x}{}^0, \pab{x}{}^1, \pab{x}{}^2, \ldots, \pab{x}{}^n)$ and $\log_nS(\tensor{Y})_{t_1,t_2}=(\pab{y}{}^0, \pab{y}{}^1, \pab{y}{}^2, \ldots, \pab{y}{}^n)$. To simplify notation, we write $\logx=\log_nS(\tensor{X})_{s_1,s_2}$ and $\logy=\log_nS(\tensor{Y})_{t_1,t_2}$. By bilinearity of the inner product and definitions of $\mathbb Z^x$,$\mathbb Z^y$ 
\begin{align}
f_{u,v}-f_{0,v} -f_{u,0} + f_{0,0}&=\left\langle \tensor{Z}^x_u - \tensor{Z}^x_0 ,\tensor{Z}^y_v - \tensor{Z}^y_0 \right\rangle\\
&=\Big\langle \int_{0}^{u}\left(\mathbb Z^x_{p}\tensorprod \logx\right)dp,\int_{0}^{v}\left(\mathbb Z^y_{q}\tensorprod \logy\right)dq\Big\rangle \\
&=\int_{0}^{u}\int_{0}^{v} \Big[f(p,q)\sum_{i=0}^{n}\langle \pab{x}{}^i,\pab{y}{}^i\rangle_{i}   +\sum_{\substack{i,j=0 \\ i\neq j}}^{n}\left\langle \mathbb Z^x_{p}\tensorprod \pab{x}{}^i,\mathbb Z^y_{q}\tensorprod \pab{y}{}^j\right\rangle \Big] dpdq \label{eq:pde_ker_proof}
\end{align}
By \Cref{prop:adjoint}
\[
\left\langle \mathbb{Z}^x_s \tensorprod x^i,\;\mathbb{Z}^y_{t} \tensorprod y^j \right\rangle
=
\begin{cases}
\displaystyle
\langle \ladj{\mathbb{Z}^y_t}{\mathbb{Z}^x_s},\;\radj{\pab{x}{}^i}{\pab{y}{}^j} \rangle,
& i \le j,\\
\displaystyle
\langle \ladj{\mathbb{Z}^x_s}{\mathbb{Z}^y_t},\;\radj{y^j}{x^i} \rangle,
& i \ge j.
\end{cases}
\]
Therefore, the last term in the integrand can be rewritten as
\begin{align*}
\sum_{\substack{i,j=0 \\ i\neq j}}^{n}\!\!\langle \mathbb Z^x_{p}\tensorprod x^i, \mathbb Z^y_{q}\tensorprod y^j\rangle &= \sum_{i=1}^{n}\sum_{j=0}^{i-1}\langle \ladj{\mathbb Z^x_s}{\mathbb Z^y_t}, \radj{\pab{y}{}^j}{\pab{x}{}^i}\rangle + \sum_{i=0}^{n}\sum_{j=i+1}^{n} \!\!\langle\ladj{\mathbb Z^y_t}{\mathbb Z^x_s}, \radj{\pab{x}{}^i}{\pab{y}{}^j } \rangle\\
\vspace{-1cm}
&=\sum_{i=1}^{n}\sum_{j=0}^{i-1}\langle \adjb(p,q), \radj{\pab{y}{}^j}{\pab{x}{}^i}\rangle + \sum_{i=0}^{n}\sum_{j=i+1}^{n} \!\!\langle\adja(p,q), \radj{\pab{x}{}^i}{\pab{y}{}^j } \rangle.
\end{align*}
where we have introduced the following $\pi_{\leq n}(T^{>0}(\statespace))$-valued state variables
\begin{align*}
    \adja(s,t) = \pi_{\leq n}(\ladj{\mathbb Z^y_t}{\mathbb Z^x_s}-f(s,t)\identity), \quad 
   \adjb(s,t) =  \pi_{\leq n}(\ladj{\mathbb Z^x_s}{\mathbb Z^y_t}-f(s,t)\identity). 
\end{align*}
Plugging this in \cref{eq:pde_ker_proof} and differentiating with respect to $u$ and $v$ we obtain the first PDE of the system. 

Now, we derive the equation for $\langle\adja,e_w\rangle$ for any non empty word $w$ of length $w\leq n$. 
\begin{align}\label{eq:adj_integral_eq}
    \adja^w_{u,v}-\adja^w_{0,v}&=\langle\adja(u,v)-\adja(0,v), e_w\rangle\\
    &=\langle \mathbb Z^x_{u}-\mathbb Z^x_{0},\mathbb Z^y_{v}\tensorprod e_w\rangle\\
    &= \Big\langle \int_{0}^{u}\left(\mathbb Z^x_{p}\tensorprod \logx\right)dp,\mathbb Z^y_{v}\tensorprod e_w\Big\rangle\\
    &= \int_{0}^{u}\langle \mathbb Z^x_{p}\tensorprod\logx,\mathbb Z^y_{v}\tensorprod e_w\rangle dp
\end{align}
We write $i=|w|$ and use \Cref{prop:adjoint} to rewrite the integrand as
\begin{align}\label{eq:adj_equation_trick}
    \langle \mathbb Z^x_{u}\tensorprod \logx,\mathbb Z^y_{v}\tensorprod e_w\rangle &= f_{u,v}\langle \logx, e_w\rangle + \sum_{j=0}^{i-1}\langle  \ladj{\mathbb Z^y_{v}}{\mathbb Z^x_{u}}, \radj{\pab{x}{}^j}{e_w}\rangle + \sum_{j=i+1}^{n} \langle  \ladj{\mathbb Z^x_{u}}{\mathbb Z^y_{v}}, \radj{e_w}{\pab{x}{}^j}\rangle
\end{align}
Plugging in \ref{eq:adj_equation_trick} in \ref{eq:adj_integral_eq} we obtain
\begin{align*}
\adja^w_{u,v}&=\adja^w_{0,v}+\int_{0}^{u}\Big[f_{p,v}\langle \logx, e_w\rangle  + 
    \sum_{j=0}^{i-1}\langle  \adja_{p,v}, \radj{\pab{x}{}^j}{e_w}\rangle + \sum_{j=i+1}^{n}\langle  \adjb
    _{p,v}, \radj{e_w}{\pab{x}{}^j}\rangle \Big]dp 
\end{align*}
Differentiating with respect to $u$ we obtain the second PDE of the system in $\statespace ^{\otimes i}$. The third PDE can be derived in the same way.

\subsection{Vectorial form}

The first equation in \Cref{thm:pde_genericn} can be written (in integral form) as
\begin{align*}
    f_{u,v}&=f_{0,v} + f_{u,0} - f_{0,0} + \int_{0}^{u}\int_{0}^{v} f(p,q)\langle \logx,\logy\rangle dpdq \\
    &\quad+\int_{0}^{u}\int_{0}^{v} \Big[\left\langle \adjb(p,q), \radj{\logy}{\logx}\right\rangle +\left\langle \adja(p,q), \radj{\logx}{\logy}\right\rangle\Big] dpdq. 
\end{align*}
We have used $\langle \adjb(s,t),\radj{\pab{y}{}^j}{\pab{x}{}^i}\rangle=0$ for $j\geq i$. This comes from $\radj{\pab{y}{}^j}{\pab{x}{}^i}=0$ (case $j>i$) and $\langle \adjb(s,t),e_{{\o}}\rangle=0$ (case $i=j$). Therefore,we could rewrite 
\begin{align*}
    \sum_{i=1}^{n}\sum_{j=0}^{i-1}\left\langle \adjb(u,v), \radj{\pab{y}{}^j}{\pab{x}{}^i}\right\rangle &= \sum_{i=0}^{n}\sum_{j=0}^{n}\left\langle \adjb(u,v), \radj{\pab{y}{}^j}{\pab{x}{}^i}\right\rangle = \left\langle \adjb(u,v), \radj{\logy}{\logx}\right\rangle. 
\end{align*}
Similarly, using $\radj{\pab{x}{}^i}{\pab{y}{}^j}=0$ for $i>j$ and $\langle \adja(s,t),e_{{\o}}\rangle=0$ we obtain

\begin{align*}
\sum_{i=0}^{n}\sum_{j=i+1}^{n}\left\langle \adja(u,v), \radj{\pab{x}{}^i}{\pab{y}{}^j}\right\rangle
&=\sum_{i=0}^{n}\sum_{j=0}^{n}\left\langle \adja(u,v), \radj{\pab{x}{}^i}{\pab{y}{}^j}\right\rangle =\left\langle \adja(u,v), \radj{\logx}{\logy}\right\rangle. 
\end{align*}
We now seek an equation for tensor $\adja$ 
\begin{align}\label{eq:adj_tensor_series}
\adja_{s,t} &= 0\cdot e_{{\o}}+\sum_{i=1}^{n}\sum_{w:|w|=i}\adja^w_{s,t}e_w 
\end{align}
Using $\langle \adja, e_{{\o}}\rangle=\langle \adjb, e_{{\o}}\rangle=0$ and $\radj{\pab{x}{}^j}{e_w}=0$ for $j>i$ and $\radj{e_w}{\pab{x}{}^j}=0$ for $j<i$
\begin{align}\label{eq:component_wise_eq_adj}
 \adja^w_{u,v} &=\adja^w_{0,v}+\int_{0}^{u}\Big[f_{p,v}\langle \logx, e_w\rangle + 
\langle  \adja_{p,v}, \radj{\logx}{e_w}\rangle + \langle  \adjb_{p,v}, \radj{e_w}{\logx}\rangle \Big]dp
\end{align}
Plugging \cref{eq:component_wise_eq_adj} in \cref{eq:adj_tensor_series} we obtain
\begin{align*}
\adja(u,v) &=  \sum_{i=1}^{n}\sum_{w:|w|=i}\adja^w(0,v)e_w + \sum_{i=1}^{n}\sum_{w:|w|=i}\left(\int_{0}^{u}\langle  \adja(p,v), \radj{\logx}{e_w}\rangle dp\right)  e_w \\
&\quad + \sum_{i=1}^{n}\sum_{w:|w|=i}\left(\int_{0}^{u}\langle  \adjb(p,v), \radj{e_w}{\logx}\rangle dp\right)  e_w.
\end{align*}
The penultimate term can be rewritten as 
\begin{align*}
\sum_{i=1}^{n}\sum_{w:|w|=i}\left(\int_{0}^{u}\langle  \adja(p,v), \radj{\logx}{e_w}\rangle dp\right) e_w &= \sum_{i=0}^{n}\sum_{w:|w|=i}\left(\int_{0}^{u}\langle  \adja(p,v), \radj{\logx}{e_w}\rangle dp\right)  e_w \\
&=\sum_{i=0}^{n}\sum_{w:|w|=i}\left(\int_{0}^{u}\langle  \adja(p,v)\tensorprod \logx,e_w\rangle dp \right) e_w \\
&= \int_{0}^{u}\left(\adja(p,v)\tensorprod\logx\right) dp
\end{align*}
where the first equality comes from the fact that $\langle \adja(s,t), e_{{\o}}\rangle=0$ and the second equality from the definition of the adjoint of right tensor multiplication. For the last term, we use the fact that $\langle \adjb(s,t), \radj{e_{{\o}}}{\logx}\rangle=\langle \adjb(s,t),\logx\rangle$, and obtain
\begin{align*}
I&=\sum_{i=1}^{n}\sum_{w:|w|=i}\left(\int_{0}^{u}\langle  \adjb(p,v), \radj{e_w}{\logx}\rangle dp\right)  e_w\\ &=\sum_{i=0}^{n}\sum_{w:|w|=i}\left(\int_{0}^{s}\langle  \adjb(p,v), \radj{e_w}{\logx}\rangle dp\right)   e_w -  \left(\int_{0}^{u}\langle  \adjb(p,v), \logx\rangle dp\right) e_{{\o}}\\
&=\sum_{i=0}^{n}\sum_{w:|w|=i}\left(\int_{0}^{u}\langle  \adjb(p,v)\tensorprod e_w, \logx\rangle dp\right)   e_w -  \left(\int_{0}^{u}\langle  \adjb(p,v), \logx\rangle dp\right) e_{{\o}}\\
&=\sum_{i=0}^{n}\sum_{w:|w|=i}\left(\int_{0}^{u}\langle e_w, \ladj{ \adjb_{p,v}}{\logx}\rangle dp\right)   e_w - \left( \int_{0}^{u}\langle  \adjb(p,t), \logx\rangle dp\right) e_{{\o}}\\
&=\int_{0}^{u}\ladj{ \adjb_{p,v}}{\logx}dp-  \left(\int_{0}^{u}\langle  \adjb(p,v), \logx\rangle dp\right) e_{{\o}}\\
\end{align*}
Finally, we get
\begin{align}   \adja(u,v)=\adja(0,v)+\int_{0}^{u}(\adja(p,v)\tensorprod \logx) dp
+ \int_{0}^{u}\ladj{ \adjb_{p,v}}{\logx}dp-  \left(\int_{\sigma}^{s}\langle  \adjb(p,v), \logx\rangle dp\right) e_{{\o}}.
\end{align}

\end{document}